\documentclass[11pt,a4paper]{article}

\usepackage[dvipdfmx]{graphicx}
\usepackage{amsmath}
\usepackage{amssymb}
\usepackage{ascmac}
\usepackage{authblk}
\usepackage{bm}
\usepackage{caption}
\usepackage{color}
\usepackage[T1]{fontenc}
\usepackage[colorlinks=true]{hyperref}
\usepackage{mathtools}
\usepackage[numbers]{natbib}
\usepackage{stmaryrd}
\usepackage{tgtermes}
\usepackage[capitalize]{cleveref}

\usepackage{amsmath}
\usepackage{amssymb}
\usepackage{amsthm}
\usepackage{bm}
\usepackage{booktabs}
\usepackage[font=footnotesize,labelfont=bf]{caption}
\usepackage{color}
\usepackage{dsfont}
\usepackage{enumitem}
\usepackage[T1]{fontenc}
\usepackage{mathtools}
\usepackage{mathrsfs}
\usepackage[framemethod=tikz]{mdframed}
\usepackage{multirow}
\usepackage{pifont}
\usepackage{rotating}
\usepackage[group-separator={,},group-minimum-digits={3}]{siunitx}
\usepackage{subcaption}
\usepackage{tabularx}
\usepackage{thm-restate}
\usepackage{tikz}
\usepackage{wrapfig}

\hypersetup{
  colorlinks=true,
  linkcolor={red!15!green!35!blue!95},
  citecolor={green!50!black},
  urlcolor={red!50!black}
}
\setlist[itemize]{itemsep=2pt,parsep=0pt,topsep=0pt,leftmargin=10pt}
\setlist[enumerate]{itemsep=2pt,parsep=0pt,topsep=0pt,leftmargin=12pt}
\allowdisplaybreaks[4]

\newmdenv[
  linecolor=blue!35,
  linewidth=1.0pt,
  backgroundcolor=blue!5,
  frametitlealignment=\center,
]{takehome}

\newtheorem*{theorem*}{Theorem}
\newtheorem*{proof*}{proof}
\newtheorem{definition}{Definition}
\newtheorem*{definition*}{Definition}
\newtheorem{assumption}{Assumption}
\newtheorem*{assumption*}{Assumption}
\newtheorem{lemma}{Lemma}
\newtheorem*{lemma*}{Lemma}

\newtheorem*{proposition*}{Proposition}

\newtheorem*{corollary*}{Corollary}
\newtheorem{remark}{Remark}

\crefname{assumption}{Assumption}{Assumptions}

\renewcommand{\tilde}{\widetilde}

\newcommand{\Ncal}{\mathcal{N}}

\newcommand{\Ebb}{\mathbb{E}}

\newcommand{\Pbb}{\mathbb{P}}

\newcommand{\Rbb}{\mathbb{R}}

\newcommand{\Vbb}{\mathbb{V}}

\newcommand{\Abf}{\mathbf{A}}
\newcommand{\Bbf}{\mathbf{B}}

\newcommand{\Fbf}{\mathbf{F}}

\newcommand{\Hbf}{\mathbf{H}}
\newcommand{\Ibf}{\mathbf{I}}

\newcommand{\Pbf}{\mathbf{P}}

\newcommand{\Rbf}{\mathbf{R}}
\newcommand{\Sbf}{\mathbf{S}}

\newcommand{\Ubf}{\mathbf{U}}

\newcommand{\Wbf}{\mathbf{W}}
\newcommand{\Xbf}{\mathbf{X}}

\newcommand{\Zbf}{\mathbf{Z}}
\newcommand{\abf}{\mathbf{a}}

\newcommand{\ebf}{\mathbf{e}}

\newcommand{\mbf}{\mathbf{m}}

\newcommand{\qbf}{\mathbf{q}}

\newcommand{\xbf}{\mathbf{x}}
\newcommand{\ybf}{\mathbf{y}}
\newcommand{\zbf}{\mathbf{z}}

\newcommand{\Gammabf}{\bm{\Gamma}}

\newcommand{\Thetabf}{\bm{\Theta}}

\newcommand{\Sigmabf}{\bm{\Sigma}}

\newcommand{\gammabf}{\bm{\gamma}}

\newcommand{\mubf}{\bm{\mu}}

\newcommand{\rhobf}{\bm{\rho}}

\newcommand{\omegabf}{\bm{\omega}}

\newcommand{\zerobf}{\mathbf{0}}
\newcommand{\onebf}{\mathbf{1}}

\newcommand{\defeq}{\coloneqq}

\DeclareMathOperator*{\E}{\Ebb}

\DeclareMathOperator{\diag}{\mathrm{diag}}
\DeclareMathOperator{\tr}{\mathrm{tr}}

\newcommand{\indicator}[1]{\mathds{1}_{\{#1\}}}
\newcommand{\set}[1]{\left\lbrace{#1}\right\rbrace}

\newcommand{\inpr}[2]{\left\langle{#1},{#2}\right\rangle}

\newcommand{\diff}[2]{\frac{d{#1}}{d{#2}}}

\renewcommand{\epsilon}{\varepsilon}

\newcommand{\erf}{{\mathop\mathrm{erf}}}

\newcommand{\frob}{{\mathrm{F}}}

\newcommand{\QK}{{\mathrm{QK}}}
\newcommand{\Val}{{\mathrm{V}}}
\newcommand{\WQK}{\Wbf_\QK}
\newcommand{\WQ}{\Wbf_{\mathrm{Q}}}
\newcommand{\WK}{\Wbf_{\mathrm{K}}}
\newcommand{\WV}{\Wbf_\Val}
\newcommand{\WF}{\Wbf_{\mathrm{F}}}

\usepackage[top=30truemm,bottom=30truemm,left=25truemm,right=25truemm]{geometry}

\title{Self-attention Networks Localize When QK-eigenspectrum Concentrates}
\author[1,$\dagger$]{Han Bao}
\author[2]{Ryuichiro Hataya}
\author[3]{Ryo Karakida}
\affil[1]{Kyoto University}
\affil[2]{RIKEN AIP}
\affil[3]{AIST}
\affil[$\dagger$]{\href{bao@i.kyoto-u.ac.jp}{bao@i.kyoto-u.ac.jp}}
\date{\today}

\begin{document}
\maketitle

\setlength{\parskip}{4pt}
\setlength{\abovedisplayskip}{5pt}
\setlength{\belowdisplayskip}{5pt}

\begin{abstract}
    The self-attention mechanism prevails in modern machine learning.
    It has an interesting functionality of adaptively selecting tokens from an input sequence by modulating the degree of attention localization, which many researchers speculate is the basis of the powerful model performance but complicates the underlying mechanism of the learning dynamics.
    In recent years, mainly two arguments have connected attention localization to the model performances.
    One is the rank collapse, where the embedded tokens by a self-attention block become very similar across different tokens, leading to a less expressive network.
    The other is the entropy collapse, where the attention probability approaches non-uniform and entails low entropy, making the learning dynamics more likely to be trapped in plateaus.
    These two failure modes may apparently contradict each other because the rank and entropy collapses are relevant to uniform and non-uniform attention, respectively.
    To this end, we characterize the notion of attention localization by the eigenspectrum of query-key parameter matrices and reveal that a small eigenspectrum variance leads attention to be localized.
    Interestingly, the small eigenspectrum variance prevents both rank and entropy collapse, leading to better model expressivity and trainability.
\end{abstract}

\section{Introduction}
\label{section:introduction}
Transformers have been widely adopted in language modeling~\citep{Vaswani2017NeurIPS}, vision tasks~\citep{Dosovitskiy2021ICLR,Touvron2021ICML}, and speech recognition~\citep{Likhomanenko2021Interspeech}.
A crucial building block in transformers is the attention mechanism, dating back to \citet{Graves2013}, which was initially designed to capture long-range signals in sequential inputs by mixing individual tokens but has also been leveraged to capture general structures of input data.
After the fully-attention-based language model has appeared \citep{Brown2020NeurIPS}, the research community gets interested in the functionality and benefits of the attention.
To mention a few, transformers implicitly prefer hierarchical interpretations of input sequences~\citep{Kharitonov2021ICLR}; store relational knowledge in MLP layers as an associative memory~\citep{Meng2022NeurIPS}; its computational graphs tend to be tree-structured~\citep{Murty2023ICLR}; suddenly capture tree structures of inputs after long training epochs~\citep{Murty2023ACL}.
Theoretically, training dynamics analysis explains how to learn spatially correlated patches by vision transformers (ViT)~\citep{Jelassi2022NeurIPS}, select dominant tokens~\citep{Tian2023NeurIPS}, store information as an associative memory~\citep{Bietti2023NeurIPS}, and select max-margin tokens~\citep{Tarzanagh2023NeurIPS}, whereas \citet{Xie2022ICLR} explains the in-context learning as a process of concept learning in Bayesian inference.

Among many aspects of attention, we specifically focus on \emph{localization}---for a query token, self-attention can select a few relevant tokens only (which we call \emph{localized} attention) or select many tokens \emph{uniformly}.
As attention can be regarded as a token mixer, it plays a pivotal role in studying how it selects tokens to reveal the characteristics of the token embeddings.
To this end, we have the following research questions:
(Q1) \emph{When} is self-attention localized or uniform?
(Q2) \emph{How} does localization affect model performances?

Along this line, previous studies mainly investigated from the model expressivity and training stability perspectives.
On the one hand, \citet{Dong2021ICML} and \citet{Noci2022NeurIPS} initiated the discussion of attention localization and theoretically showed that a network with self-attention layers without skip connections exponentially loses the rank of hidden layers; the fact indicates that \emph{the model expressivity shall be immediately lost with more self-attention layers stacked}.
On the other hand, \citet{Zhai2023ICML} empirically found that attention entropy---averaged Shannon entropy of an attention probability matrix---correlates with training stability.
Specifically, a training loss curve tends to fall into a plateau when attention entropy is low.
Since higher entropy indicates near-uniform attention weights, their finding apparently suggests that \emph{localized attention may lead the learning dynamics to a plateau}.
Up until now, these two failure modes have been discussed independently with slightly different notions of attention localization, and hence, our understanding of the blessing and curse of attention localization remains elusive.

To better comprehend, we characterize self-attention patterns by attention parameter matrices to reconcile the two collapse modes.
We formulate the concept of localization by \emph{signal propagation probability} (\cref{section:direction}), which describes how likely the signal of a specific input token propagates to the gradient of a training objective.
If the signal propagation probability is high for a few numbers of tokens only, attention is regarded to be localized.
We show that the localization mode can be characterized by the eigenspectrum of attention weight matrices (\cref{section:analysis}).
Specifically, attention is localized in the above sense when the eigenspectrum of the query-key parameter matrix has a non-zero mean and a small variance.
Furthermore, the small eigenspectrum variance is relevant to both the rank collapse and entropy collapse (\cref{section:related}), and thus, we give a unified perspective of the two notions of attention collapse.
For this reason, we argue that attention collapse and its performance can be viewed more transparently based on the eigenspectrum variance.
Lastly, we verified the correlation of the eigenspectrum and the model performance in the experiments with the WikiText dataset~\citep{Merity2016} by introducing a regularization scheme called \textsc{LocAteR}.

\begin{figure}
    \centering
    \begin{minipage}{0.23\textwidth}
        \includegraphics[width=\textwidth]{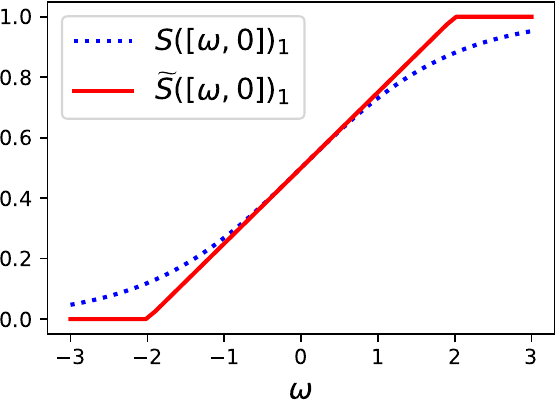}
    \end{minipage}
    \begin{minipage}{0.23\textwidth}
        \includegraphics[width=\textwidth]{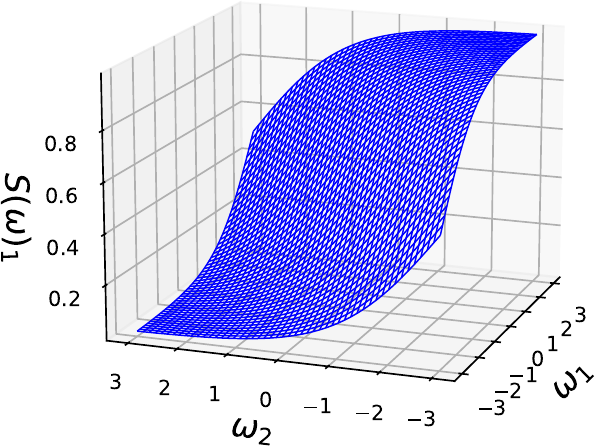} \\
        \includegraphics[width=\textwidth]{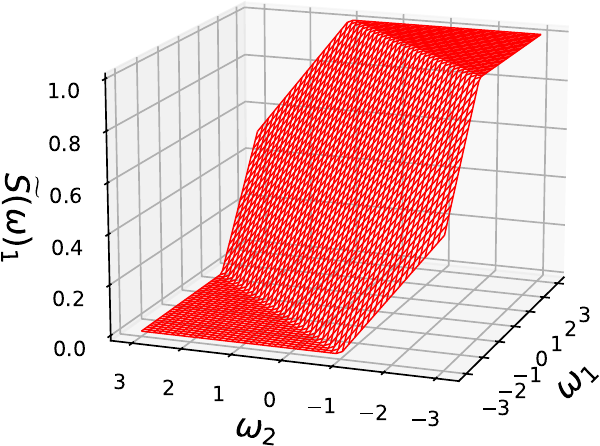}
    \end{minipage}
    \caption{Comparison of softmax $\Sbf$ and the piecewise approximation $\tilde\Sbf$ for two-dimensional inputs.}
    \label{figure:softmax}
\end{figure}

\section{Setup}
\label{section:setup}

\paragraph{Notation.}
We write vectors with all zeros and ones by $\zerobf$ and $\onebf$, respectively, whereas the $i$-th one-hot vector is written as $\ebf^i$.
A vector is written in bold-face like $\abf$, and its $i$-th scalar element is written by non-bold $a_i$.
A matrix is written by capital bold-face like $\Abf$, and $\Abf_i$ denotes its $i$-th column vector unless otherwise noted.
The identity matrix is denoted by $\Ibf$.
The Hadamard product of $\Abf$ and $\Abf$ is written by $\Abf^{\odot2} \defeq \Abf \odot \Abf$.
We write the error function as $\erf$ and use $\erf(-z) = -\erf(z)$ without explicitly mentioning it.
Infinitesimal asymptotic orders $o(\cdot)$ are with respect to the sequence length $T$ unless otherwise noted.

\paragraph{Transformer.}
Let $\Xbf \defeq [\xbf_1 \; \xbf_2 \; \dots \; \xbf_T] \in \Rbb^{d \times T}$ be an input with $T$ tokens, defined later shortly.
We suppose that all input sequences have the same length $T$, and $T$ is occasionally taken sufficiently large.
The $\ell$-th (single-head) self-attention layer is defined as
\begin{align}
    \label{equation:attention_mask}
    \Abf^\ell &\defeq \Sbf\bigg(\frac{(\Xbf^{\ell-1})^\top\WQK\Xbf^{\ell-1}}{\lambda}\bigg), \\
    \label{equation:self_attention}
    \Ubf^\ell &\defeq \WV\Xbf^{\ell-1}\Abf^\ell,
\end{align}
where $\WV \in \Rbb^{d \times d}$ is the value parameters, $\WQK (\defeq \WQ^\top\WK) \in \Rbb^{d \times d}$ is the query-key parameters (with joint parametrization), $\lambda > 0$ is temperature, commonly $\lambda = \sqrt{d}$, and $\Sbf: \Rbb^T \to \Rbb^T$ is the softmax applied for each row.
In this way, each input token in $\Xbf^{\ell-1}$ (embedded by $\WV$) is mixed by $\Abf^\ell$.
Then, the transformer block (without layer normalization \cite{Ba2016}) is defined as
\begin{align*}
    & \Zbf^\ell \defeq \Ubf^\ell + \Xbf^{\ell-1}, \\
    & \Hbf^\ell \defeq \Wbf_{\mathrm{F}_2}\sigma(\Wbf_{\mathrm{F}_1}\Zbf^\ell), \\
    & \Xbf^\ell \defeq \Hbf^\ell + \Zbf^\ell,
\end{align*}
where $\Hbf^\ell$ is a feed-forward net with parameters $\Wbf_{\mathrm{F}_1}, \Wbf_{\mathrm{F}_2} \in \Rbb^{d \times d}$ and an (element-wise) activation $\sigma: \Rbb \to \Rbb$.
We omit the token embedding layer and set $\Xbf^0 \defeq \Xbf$.
There are two common variants of layer normalization positions, Post-LN \cite{Vaswani2017NeurIPS} and Pre-LN \cite{Xiong2020ICML}, which are applied token-wise after the residual connections ($\Zbf^\ell$ and $\Xbf^{\ell+1}$) and before the inputs ($\Xbf^\ell$ and $\Zbf^\ell$), respectively.
Then, the transformer block $\Xbf^\ell$ is stacked $L$ times and $\Fbf(\Xbf) \defeq \Xbf^{L} \in \Rbb^{d \times T}$ is the output.

\paragraph{Learning task.}
We focus on causal language modeling, where a model predicts the next token given contextual tokens.
Formally, given $T$ contextual tokens $\Xbf \in \Rbb^{d \times T}$, the prediction target is the $(T+1)$-th token $\ybf \defeq \xbf_{T+1} \in \Rbb^d$.
With the squared loss, the objective is written as follows:
\[
    J(\Thetabf) \defeq \frac{1}{2}\E\|\ybf - \Fbf(\Xbf)_T\|^2,
\]
where $\Thetabf \defeq (\WV, \WQK, \Wbf_{\mathrm{F}_1}, \Wbf_{\mathrm{F}_2})$ denotes the model parameter set, and the expectation is taken over input sequences $(\Xbf, \ybf)$.
Here, our decoding procedure in consideration is to simply choose the embedded last token $\Fbf(\Xbf)_T \in \Rbb^T$.
The parameters $\Thetabf$ are learned by minimizing $J$.
Note that our analysis considers optimizing the query-key parameters jointly.
Although such joint parameterization is less common in practice, it is convenient for the theoretical derivation of the gradients and has been used in several previous studies \cite{Jelassi2022NeurIPS,Tian2023NeurIPS}.
Interested readers may refer to a recent work revealing that the joint and separate QK-parametrization lead to different implicit regularizations \citep{Tarzanagh2023}.

\paragraph{Picewise linear approximation of softmax.}
In this article, we choose to approximate the softmax function $\Sbf$ by linearization.
For a $T$-dimensional input $\omegabf \in \Rbb^T$, the softmax function is defined as
\[
    S(\omegabf)_i \defeq \frac{\exp(\omega_i)}{\sum_{j \in [T]} \exp(\omega_j)} \text{~~~for all } i \in [T].
\]
For linearization, the Taylor expansion of $S(\omegabf)_i$ around the origin $\inpr{\gammabf^i}{\omega} + \gamma_0^i$ is used, where
\[
    \gammabf^i \defeq \nabla_i\Sbf(\zerobf) = \frac{1}{T}\ebf^i - \frac{1}{T^2}\onebf, \quad
    \gamma_0^i \defeq S(\zerobf)_i = \frac{1}{T}.
\]
Then, we approximate $\Sbf$ by the piecewise linear function such that $S(\omegabf)_i \approx \max\{0, \min\{1, \inpr{\gammabf^i}{\omegabf} + \gamma_0^i\}\} = \inpr{\tilde\gammabf^i}{\omegabf} + \tilde\gamma_0^i$, where
\[
    (\tilde\gammabf^i, \tilde\gamma_0^i) = \begin{cases}
        (\zerobf, 0) & \text{if } \inpr{\gammabf^i}{\omegabf} + \gamma_0^i < 0, \\
        (\gammabf^i, \gamma_0^i) & \text{if } \inpr{\gammabf^i}{\omegabf} + \gamma_0^i \in [0, 1], \\
        (\zerobf, 1) & \text{if } \inpr{\gammabf^i}{\omegabf} + \gamma_0^i > 1. \\
    \end{cases}
\]
In the vector form, the piecewise approximation $\Sbf(\omegabf) \approx \tilde\Sbf(\omegabf)$ is given by
\[
    \tilde\Sbf(\omegabf) = \Gammabf^\top\omegabf + \tilde\gammabf_0, \text{~~where~}
    \begin{cases}
        \Gammabf \defeq [\tilde\gammabf^1 \; \tilde\gammabf^2 \; \dots \tilde\gammabf^T], \\
        \tilde\gammabf_0 = [\tilde\gamma_0^1, \tilde\gamma_0^2, \dots, \tilde\gamma_0^T]^\top.
    \end{cases}
\]
For notational simplicity, the column vectors of $\Gammabf$ are exceptionally denoted by $\tilde\gammabf^i$ with superscripts, for which the $\alpha$-th element is written by $\tilde\gamma^i_\alpha$.
The difference between $\Sbf$ and $\tilde\Sbf$ is illustrated in \cref{figure:softmax}.
Note that a popular alternative to softmax, sparsemax \citep{Martins2016ICML}, is also a piecewise linear function, although the functional form is slightly different from $\tilde\Sbf$.

\begin{remark}
    \label{remark:gamma_expectation}
    Each $(\tilde\gammabf^i, \tilde\gamma_0^i)$ depends on the softmax input $\omegabf$, unlike the Taylor coefficient $(\gammabf^i, \gamma_0^i)$ being independent from $\omegabf$.
    This point matters particularly when we take expectations of terms involving $(\tilde\gammabf^i, \tilde\gamma_0^i)$.
\end{remark}

\begin{remark}
    \label{remark:gamma_asymp}
    When $T$ is sufficiently large, the coefficient vector $\gammabf^i = T^{-1}\ebf^i + o(1)$.
    In this regime, $\gamma^i_\alpha = o(1)$ for any $\alpha \in [T] \setminus \set{i}$, so $\gammabf^i$ behaves like a selector of the $i$-th input.
    Hence, $\gamma^i_i = \gamma_0^i = T^{-1} + o(1)$.
\end{remark}

\section{Signal propagation probability}
\label{section:direction}

We analyze how much each token contributes to the learning dynamics.
To this end, we formalize how much the signal of a specific input token $\xbf_i$ propagates to the gradient $\nabla J$.
Remark that this notion is slightly different from the contribution of an input token $\xbf_i$ to the model output $\Fbf(\Xbf)_j$ analyzed by \citet{Kobayashi2023} recently.

\paragraph{Uniform vs. localized softmax.}
The piecewise linear approximation implies that the $i$-th input signal is propagated to the subsequent blocks when $\inpr{\gammabf^i}{\omegabf} + \gamma_0^i \in [0,1]$;
otherwise, $\tilde{S}(\omegabf)_i = \inpr{\tilde\gammabf^i}{\omegabf} + \tilde\gamma_0^i = \tilde\gamma_0^i$, which hinders the input token $\xbf_i$ from contributing to the self-attention layer~\eqref{equation:self_attention}.
Thus, we will focus on the following quantity.
\begin{definition}[Signal propagation probability]
    \label{definition:signal_propagation_probability}
    Suppose that $\WQK$ is independent from $\Xbf$.
    For $i \in [T]$, the \emph{signal propagation probability} of the $i$-th token is defined as follows:
    \[
        \rho_i \defeq \Pbb\left\{ \inpr{\gammabf^i}{\omegabf} + \gamma_0^i \in [0,1] \right\},
    \]
    where $\omegabf \defeq \Xbf^\top\WQK\xbf_T / \lambda$ and the randomness originates solely from the input tokens $\Xbf$.
\end{definition}
When only a few $\rho_i$ are significantly larger than zero, we can interpret it as localized softmax; in this case, the self-attention~\eqref{equation:self_attention} is dominated by a small number of tokens.
By contrast, most of the tokens contribute to self-attention almost equally if $\rho_i$ takes a similar value across different $i$; this situation is interpreted as uniform softmax.

\begin{figure*}
    \centering
    \includegraphics[width=0.28\textwidth]{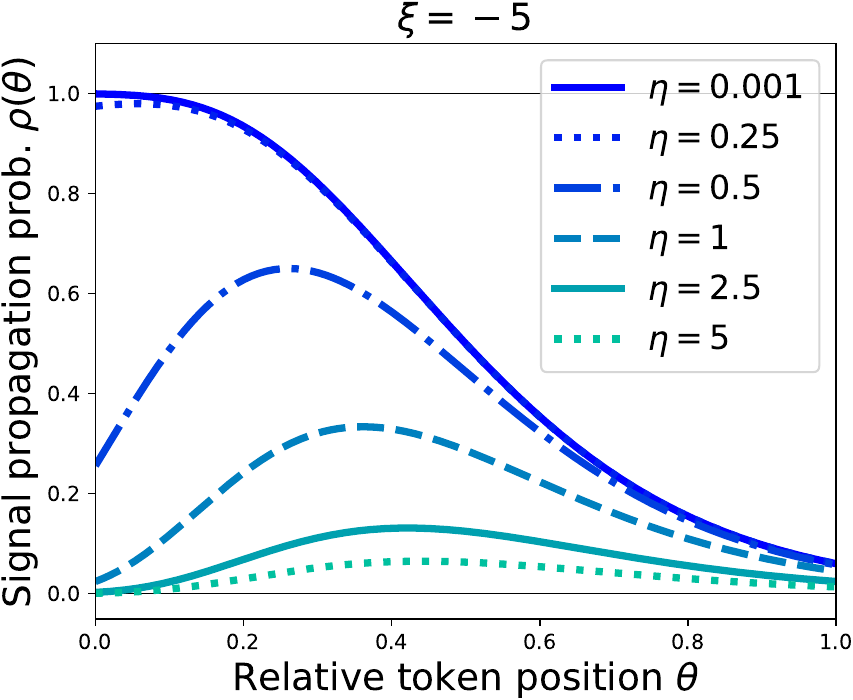}
    \includegraphics[width=0.28\textwidth]{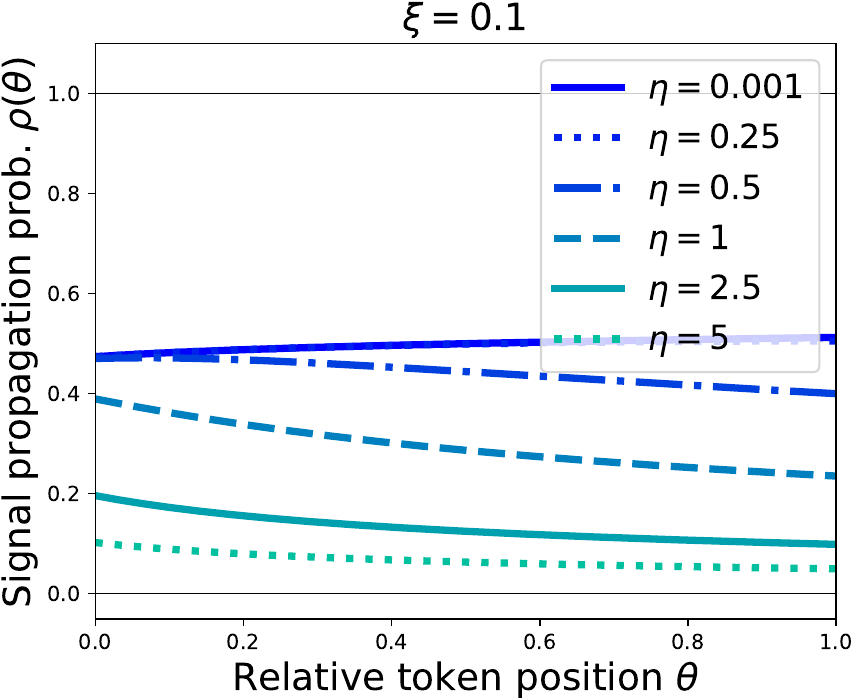}
    \includegraphics[width=0.28\textwidth]{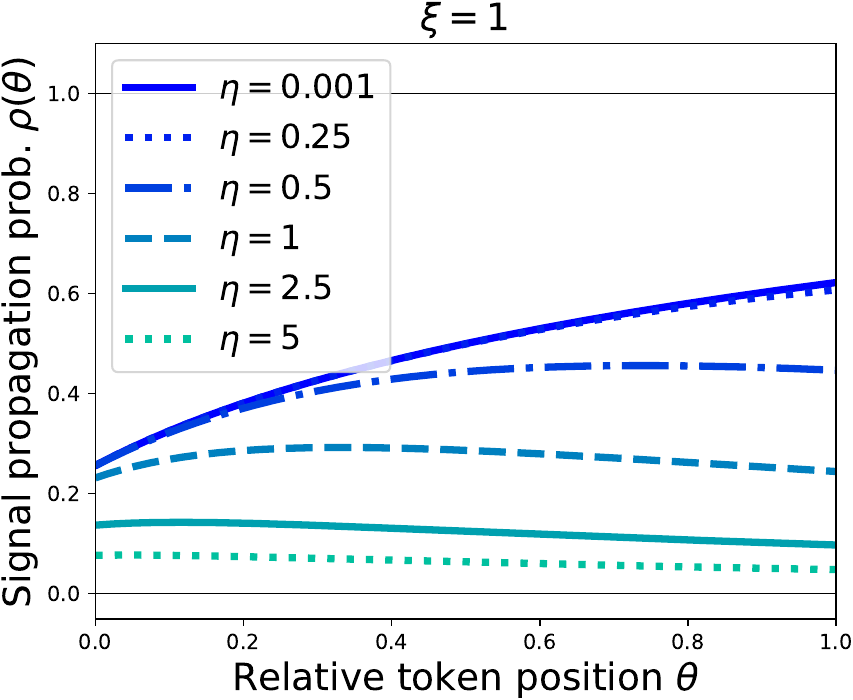} \\
    \includegraphics[width=0.28\textwidth]{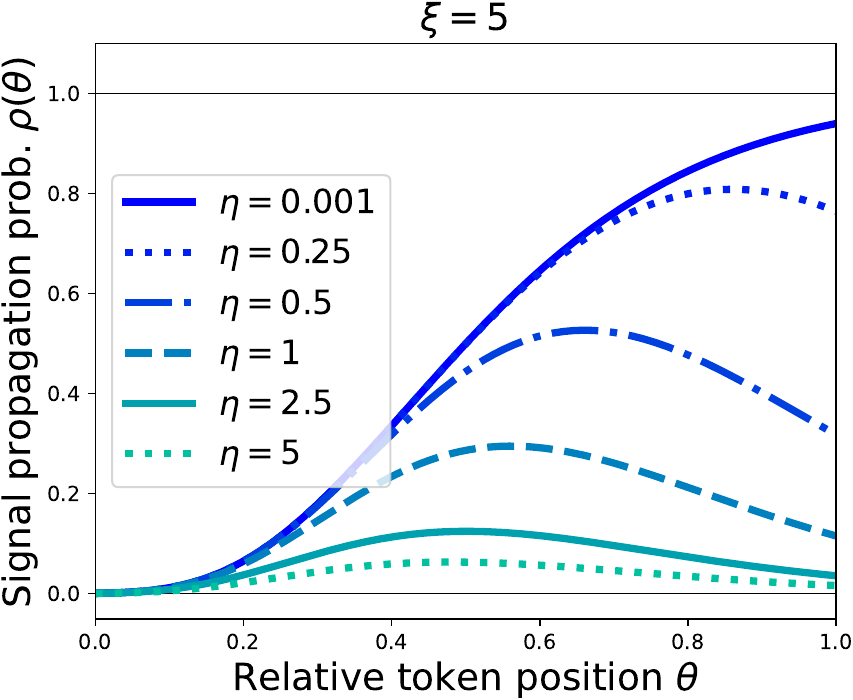}
    \includegraphics[width=0.28\textwidth]{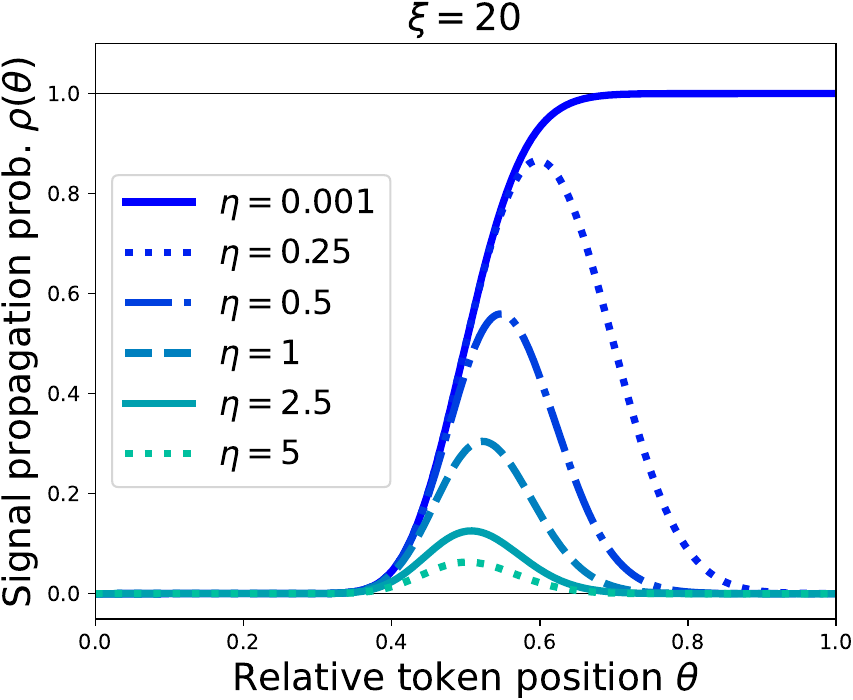}
    \includegraphics[width=0.28\textwidth]{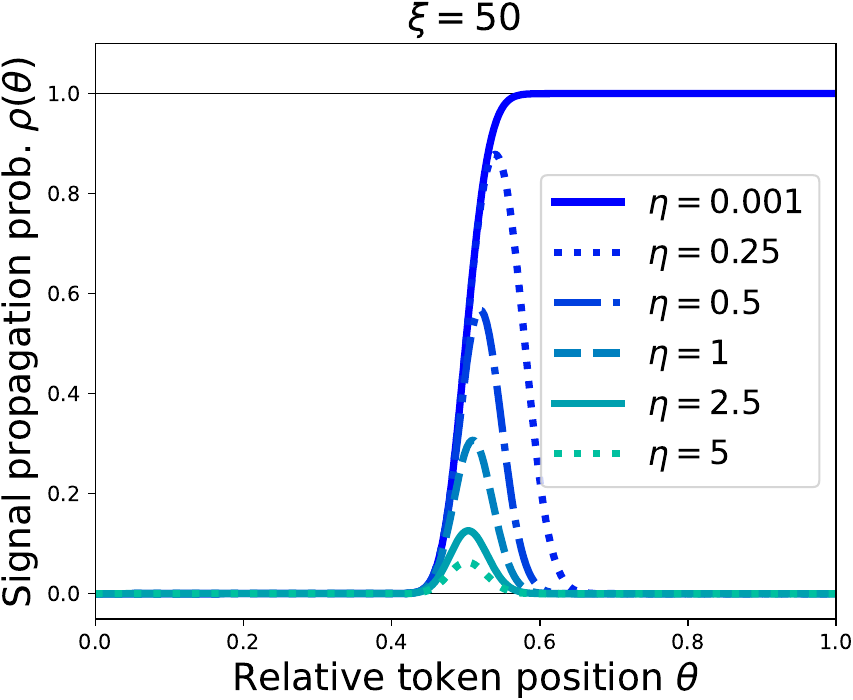}
    \caption{
        The theoretical plots of the signal propagation probability $\rho(\theta)$ with different $\xi = \tr(\Wbf)/\sqrt{\tr(\Wbf^2)}$ and $\eta = \sqrt{\tr(\Wbf^2)}/\lambda^2$.
        The vertical axes indicate relative token position $\theta = i/T$ ($i$: token index, $T$: number of tokens).
        Smaller $\theta$ close to zero and larger $\theta$ close to one correspond to early-site and late-site tokens, respectively.
    }
    \label{figure:spp}
\end{figure*}

\paragraph{Through the lens of gradient.}
The signal propagation probability naturally arises in the gradient.
Since the learning dynamics of causal language modeling is governed by the gradient flow of $J$, we can benefit from deriving the gradient of $J$ to see how attention affects the learning dynamics.
To keep the derivation concise, we consider a $1$-layer transformer (where we drop the superscripts $\ell$) without layer normalization and simplify the feed-forward net $\Hbf$ by supposing the identity activation.
With the approximated softmax $\tilde\Sbf$, the transformer can be written as follows:
\[
    \begin{aligned}
        \Fbf(\Xbf)_T
        &= \WF\{\WV\Xbf\tilde\Sbf(\omegabf) + \xbf_T\} \\
        &= \WF\{\WV\Xbf\Gammabf^\top\omegabf + \WV\Xbf\tilde\gammabf_0 + \xbf_T\},
    \end{aligned}
\]
where $\WF \defeq \Wbf_{\mathrm{F}_2}\Wbf_{\mathrm{F}_1} + \Ibf$ and $\omegabf \defeq \Xbf^\top\WQK\xbf_T / \lambda$.
For this architecture, the QK-gradient is computed:
\begin{equation}
    \label{equation:qk_grad}
    \begin{aligned}
        \nabla_{\WQK}J
        &= \lambda^{-2}\E[\Xbf\Gammabf\Pbf\Gammabf^\top\Xbf^\top\WQK\xbf_T\xbf_T^\top] \\
        &\phantom{=} + \lambda^{-1}\E[\Xbf\Gammabf\Pbf\tilde\gammabf_0\xbf_T^\top]
        + \lambda^{-1}\E[\Xbf\Gammabf\qbf\xbf_T^\top],
    \end{aligned}
\end{equation}
where
\[
    \begin{aligned}
        \Pbf &\defeq \Xbf^\top\WV^\top\WF^\top\WF\WV\Xbf, \\
        \qbf &\defeq \Xbf^\top\WV^\top\WF^\top(\WF\xbf_T - \ybf).
    \end{aligned}
\]
When $T$ is sufficiently large, we can drop asymptotically negligible terms with respect to $T$ as detailed in \cref{section:derivation}, and the QK-gradient \eqref{equation:qk_grad} is simplified as follows:
\begin{equation}
    \label{equation:qk_grad_1}
    \frac{1}{\lambda^2} \!\! \sum_{i,j,\alpha,\beta \in [T]} \!\! \E\left[ \tilde\gamma^i_\alpha\tilde\gamma^j_\beta (\xbf_i^\top\check\Pbf\xbf_j)(\xbf_\beta^\top\WQK\xbf_T)\xbf_\alpha\xbf_T^\top \right],
\end{equation}
where $\check\Pbf \defeq \WV^\top\WF^\top\WF\WV$.

Now, in the gradient term~\eqref{equation:qk_grad_1}, the summands with $\tilde\gamma^i_i$ are asymptotically dominant over those with $\tilde\gamma^i_\alpha$ (with $\alpha \ne i$) because $\tilde\gamma^i_i = T^{-1}$ and $\tilde\gamma^i_\alpha = o(T^{-1})$.
Additionally, the $i$-th signal propagates when $\tilde\gamma^i_i > 0$, which holds iff $\inpr{\gammabf^i}{\omegabf} + \gamma_0^i \in [0,1]$ by definition.
Therefore, we are motivated to check the condition $\inpr{\gammabf^i}{\omegabf} + \gamma_0^i \in [0,1]$ to see whether $\xbf_i$ contributes to the gradient~\eqref{equation:qk_grad}.
The signal propagation probability $\rho_i$ characterizes its strength.

\paragraph{Summary.}
In this section, we introduced the signal propagation probability $\rho_i$, which characterizes how likely a given token $\xbf_i$ contributes to the learning dynamics.
Specifically, $\tilde\gammabf^i \ne \zerobf$ holds more likely with larger $\rho_i$, where $\xbf_i$ contributes to the QK-gradient~\eqref{equation:qk_grad}.
Subsequently, we will analyze the quantity $\rho_i$ to see the behavior of the probability vector $\rhobf \in [0,1]^T$.
When does the mass of $\rhobf$ concentrate to only a few tokens or scatter across most of the tokens?

\section{When does attention localize?}
\label{section:analysis}

We derive the signal propagation probability $\rho_i$ based on the following synthetic data model for the sake of clarity.
\begin{assumption}[Random walk]
    \label{assumption:random_walk}
    The tokens $(\xbf_t)_{t \ge 1}$ are generated by the following Gaussian random walk:
    \[
        \xbf_1 \sim \Ncal(\zerobf, \Sigmabf), \quad
        \xbf_{t+1} \sim \Ncal(\xbf_t, \Sigmabf).
    \]
\end{assumption}
To derive $\rho_i$, we resort to the Gaussian approximation of $\inpr{\gammabf^i}{\omegabf} + \gamma_0^i$.
Define
\[
    \mu^i \defeq \E[\inpr{\gammabf^i}{\omegabf} + \gamma_0^i] \text{~~and~~}
    v^i \defeq \Vbb[\inpr{\gammabf^i}{\omegabf} + \gamma_0^i].
\]
We approximately suppose that $\inpr{\gammabf^i}{\omegabf} + \gamma_0^i \sim \Ncal(\mu^i, v^i)$.
Then, the signal propagation probability is approximated:
\begin{equation*} 
    \rho_i \approx \frac{1}{2}\left\{\erf\left(\frac{1-\mu^i}{\sqrt{2v^i}}\right) + \erf\left(\frac{\mu^i}{\sqrt{2v^i}}\right)\right\}.
\end{equation*}
To leverage this formula, we derive $\mu^i$ and $v^i$.
\begin{restatable}{lemma}{sppmeanvar}
    \label{lemma:spp_mean_var}
    Suppose that $\WQK$ is symmetric and independent from $\Xbf$, and let $\Wbf \defeq \WQK\Sigmabf$.
    Under \cref{assumption:random_walk},
    for $i \in [T]$, the mean $\mu^i$ and variance $v^i$ of $\inpr{\gammabf^i}{\omegabf} + \gamma_0^i$ with the input $\omegabf \defeq \Xbf^\top\WQK\xbf_T / \lambda$ are given as follows:
    \[
        \begin{aligned}
            \mu^i &= \left(\frac{i}{T} - \frac{1}{2}\right)\frac{\tr(\Wbf)}{\lambda} + o(1), \\
            v^i &= \left(\frac{2i^2}{T^2} + \frac{7}{12}\right)\frac{\tr(\Wbf^2)}{\lambda^2} + o(1).
        \end{aligned}
    \]
\end{restatable}
The proof is given in \cref{section:proofs}.
The symmetry of $\WQK$ is assumed for convenience.
In the case of asymmetric $\WQK$, we can redefine the signal propagation probability with the symmetrized matrix $(\WQK + \WQK^\top)/2$.

Recall that $\tr(\Wbf) = \sum_{i \in [d]} w_i$ if we write the eigenvalues of $\Wbf$ by $(w_1, w_2, \dots, w_d)$, and that $\tr(\Wbf)^2 \le d\tr(\Wbf^2)$ holds due to Jensen's inequality.
This implies that
\begin{equation}
    \label{equation:tr_tr2_inequality}
    -\sqrt{d\tr(\Wbf^2)} \le \tr(\Wbf) \le \sqrt{d\tr(\Wbf^2)}.
\end{equation}
Moreover, $\mu^i$ and $v^i$ are determined by the relative token location $i/T$.
By continuously extending $i/T$ to $\theta \in [0,1]$, the signal propagation probability $\rho_i$ can be extended to $\rho: [0,1] \to [0,1]$, defined over relative token locations:
\begin{equation}
    \label{equation:spp_approx}
    \rho(\theta) \defeq \Phi\bigg(\Big(\theta-\frac{1}{2}\Big)\xi; \theta\bigg) - \Phi\bigg(\Big(\theta-\frac{1}{2}\Big)\xi - \frac{1}{\eta}; \theta\bigg),
\end{equation}
where
{
\setlength{\abovedisplayskip}{0pt}
\[
    \begin{aligned}
        &\xi \defeq \frac{\tr(\Wbf)}{\sqrt{\tr(\Wbf^2)}}, \quad
        \eta \defeq \frac{\sqrt{\tr(\Wbf^2)}}{\lambda}, \\
        &\Phi(z;\theta) \defeq \frac{1}{2}\erf\bigg(\frac{z}{\sqrt{\smash[b]{2(2\theta^2+\frac{7}{12})}}}\bigg),
    \end{aligned}
\]
}%
and the parameter ranges are $\xi \in [-\sqrt{d}, \sqrt{d}]$ (due to the bound~\eqref{equation:tr_tr2_inequality}) and $\eta \in (0,\infty)$.
Here, $\xi$ and $\eta$ can be regarded independent (when $\Wbf$ is independent from $\Xbf$) because the eigenspectrum scale $\tr(\Wbf^2)$ can be modulated within the bound~\eqref{equation:tr_tr2_inequality} once the eigenspectrum of $\Wbf$ is given.

\Cref{figure:spp} numerically illustrates $\rho(\theta)$ with different $\xi$ and $\eta$.
From these figures, we obtain a couple of observations.
\begin{itemize}
    \item \textit{Localization.}
    $\rho(\theta)$ concentrates on fewer tokens as $\eta$ increases (see $|\xi| \ge 5$).
    By contrast, $\rho(\theta)$ behaves relatively uniformly regardless of $\eta$ for small $|\xi| \le 1$.

    \item \textit{Late-/middle-/early-site focus.}
    Focus on small $\eta$ such as $\eta = 0.001$.
    As $\xi$ increases to a large positive, $\rho(\theta)$ puts positive weights for only late-site tokens, i.e., $\theta > 0.5$.
    By contrast, as $\xi$ decreases to a negative, $\rho(\theta)$ focuses on early-site tokens, i.e., $\theta < 0.5$.
    When $\eta$ increases (see $\eta \ge 0.5$), $\rho(\theta)$ localizes around $\theta = 0.5$ with sufficiently large $\xi$ (say, $|\xi| \ge 5$), which indicates middle-site focus.

    \item \textit{Vanishing signal.}
    As $\eta$ increases, $\rho(\theta)$ degenerates to zero for any $\theta \in [0,1]$ regardless of $\xi$.
\end{itemize}

\begin{figure}
    \centering
    \includegraphics[width=0.23\textwidth]{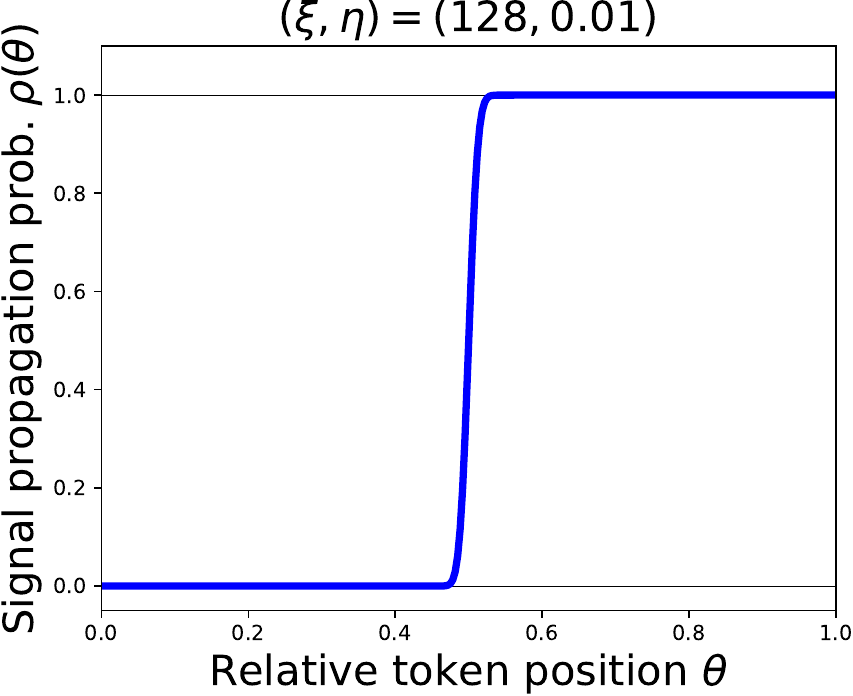}
    \includegraphics[width=0.23\textwidth]{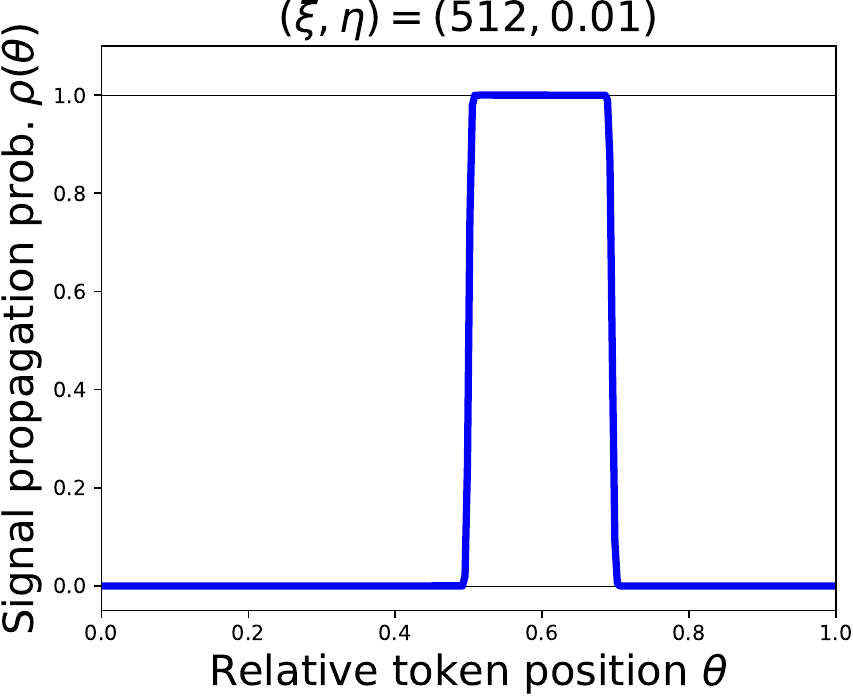}
    \caption{
        The theoretical plots of $\rho(\theta)$.
        For each $\xi = 128, 512$, the product value $\xi\eta = 1.28, 5.12$, respectively.
        The latter is sufficiently larger than the localization threshold $r = 2$ and localized.
    }
    \label{figure:localization}
\end{figure}

\begin{figure}
    \centering
    \includegraphics[width=0.4\textwidth]{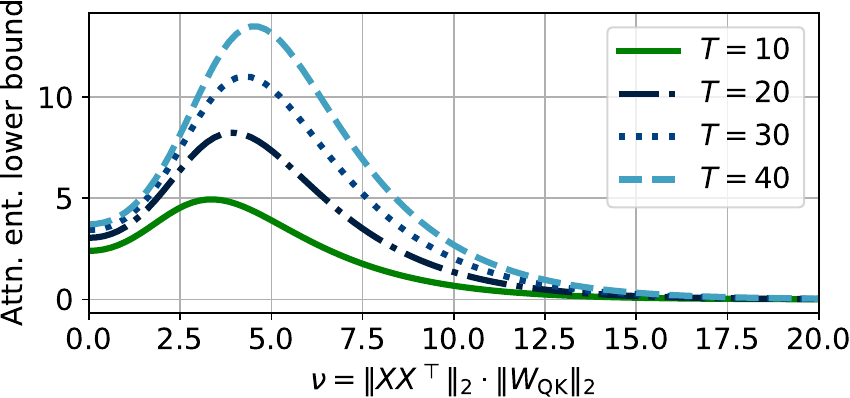}
    \caption{Entropy lower bound \eqref{equation:entropy_lower_bound} by \citet{Zhai2023ICML}.}
    \label{figure:ent_lower_bound}
\end{figure}

\paragraph{How $\rho$ behaves at the limit.}
Subsequently, we claim the above observations formally, which is proved in \cref{section:proofs}.
\begin{restatable}{lemma}{sppprop}
    \label{lemma:spp_property}
    $\rho(\theta)$ satisfies the following properties.
    \begin{enumerate}
        \item \textit{(Late-/middle-site)}
        As $(\xi,\eta) \to (\infty,0)$ with $\xi\eta \to r$,
        \[
            \rho(\theta) \to \begin{cases}
                \indicator{\theta \ge \frac{1}{2}} & \text{if } 0 \le r \le 2 \\
                \indicator{\frac{1}{2} \le \theta \le \frac{1}{2}+\frac{1}{r}} & \text{if } r > 2
            \end{cases}
            .
        \]

        \item \textit{(Early-/middle-site)}
        As $(\xi,\eta) \to (-\infty,0)$ with $\xi\eta \to r$,
        \[
            \rho(\theta) \to \begin{cases}
                \indicator{\theta \le \frac{1}{2}} & \text{if } -2 \le r < 0 \\
                \indicator{\frac{1}{2}+\frac{1}{r} \le \theta \le \frac{1}{2}} & \text{if } r < -2
            \end{cases}
            .
        \]

        \item \textit{(Uniformity)}
        Fix $\eta$ as a finite value.
        As $|\xi| \to 0$, $|\rho'(\theta)| \to 0$ for any $\theta \in [0,1]$.

        \item \textit{(Vanishing signal)}
        Fix $\xi$ as a finite value.
        As $\eta \to \infty$, $\rho(\theta) \to 0$ for any $\theta \in[0,1]$.
    \end{enumerate}
\end{restatable}
From late-/middle-/early-site focus in \cref{lemma:spp_property}, we see interestingly that $\rho(\theta)$ localizes when $\xi\eta = \tr(\Wbf)/\lambda$ asymptotically deviates from zero significantly so that $|r| \gg 2$.
At this limit, $\rho(\theta)$ concentrates on $\theta = 0.5$, inducing the middle-site focus.
Conversely, attention becomes relatively uniform when $\xi\eta = \tr(\Wbf)/\lambda$ is kept close to zero.
In \cref{figure:localization}, we numerically illustrate this regime: $\rho$ localizes at the middle site when $\eta = 0.02$ (i.e., $\xi\eta = 5.12$).

Let us investigate the limiting condition of $(\xi,\eta)$ for localization:
When is $(\xi,\eta)$ close to the limit $(\infty,0)$ while $\xi\eta \to r \gg 2$?
Here, we focus on the \emph{eigenspectrum} of $\Wbf$ by regarding its eigenvalues $(w_i)_{i\in[d]}$ as being sampled from a distribution with the mean $\tr(\Wbf) = \sum_{i\in[d]}w_i$ and scale $\tr(\Wbf^2) = \sum_{i\in[d]}w_i^2$.
First, $\eta \to 0$ indicates that the scale $\tr(\Wbf^2)$ should be close to zero.
Next, $\xi\eta (=\tr(\Wbf)/\lambda) \to r \gg 2$ means that $\tr(\Wbf) \gg 2\lambda$ at the limit, i.e., $\tr(\Wbf)$ should be significantly away from zero.
By combining them, we tell that $\rho$ localizes when the eigenspectrum concentrates around a non-zero mean.
This happens more likely when the embedding dimension $d$ is excessively large to make the eigenvalue sum $\tr(\Wbf)$ bounded away from zero while keeping the scale $\tr(\Wbf^2)$ close to zero (i.e., keeping every eigenvalue close to zero).
Thus, a larger embedding dimension $d$ is beneficial to drive attention to localization.

Next, from the claim of uniformity in \cref{lemma:spp_property}, we tell that $\rho$ fluctuates less and less with $\xi$ closer to zero.
Hence, $\rho(\theta)$ takes a similar value across different token positions $\theta$ in this limit.
When $\tr(\Wbf) \to 0$, $\rho$ attains this limit.

\paragraph{Summary.}
Wrapping up this section, we obtain answers to (Q1) posed in \cref{section:introduction} under the random walk model.
\begin{takehome}[frametitle={A1: When does attention localize?}]
    \begin{itemize}
        \item $\rho$ \emph{localizes} when $\tr(\Wbf^2)$ is close to zero while $|\tr(\Wbf)|$ is significantly bounded away from zero, i.e., $\Wbf$-eigenspectrum \emph{concentrates to a non-zero mean}.
        \item $\rho$ is \emph{uniform} when $\tr(\Wbf)$ is close to zero while $\tr(\Wbf^2)$ remains finite, i.e., $\Wbf$-eigenspectrum has \emph{the zero mean and a finite variance}.
        \item $\rho$ \emph{degenerates} to zero uniformly when $\tr(\Wbf^2)$ is sufficiently large, i.e., $\Wbf$-eigenspectrum has \emph{an infinitely large variance}.
    \end{itemize}
\end{takehome}

\section{Are different collapse regimes reconcilable?}
\label{section:related}

\begin{figure*}
    \centering
    \includegraphics[width=0.24\textwidth]{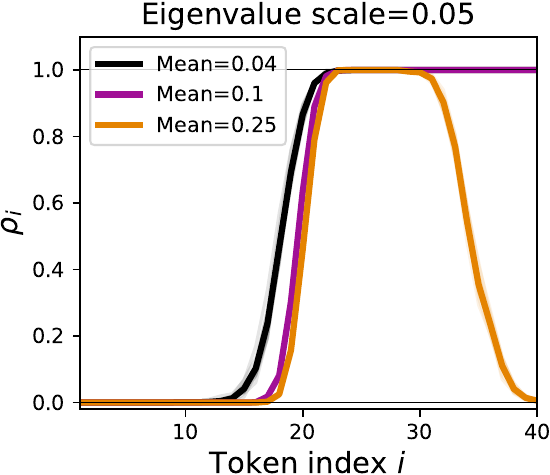} \hfill
    \includegraphics[width=0.24\textwidth]{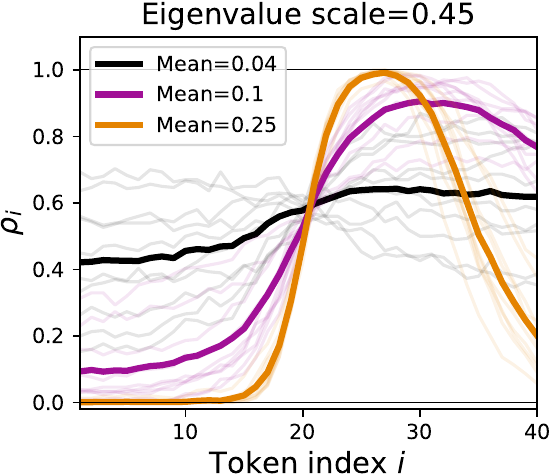} \hfill
    \includegraphics[width=0.24\textwidth]{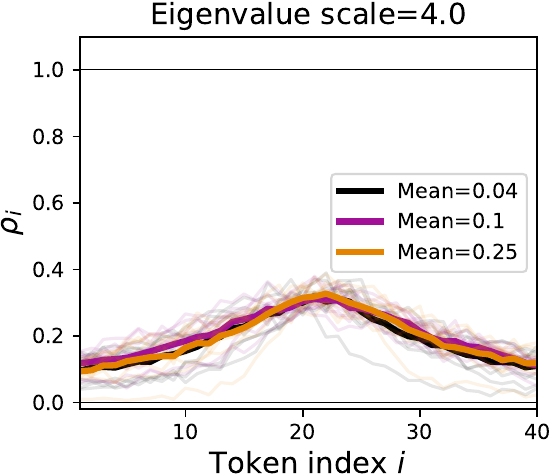} \hfill
    \includegraphics[width=0.24\textwidth]{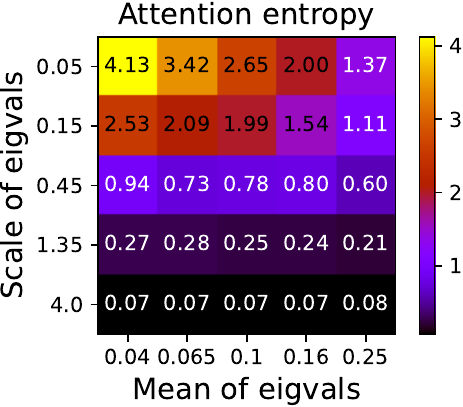} \\
    \includegraphics[width=0.24\textwidth]{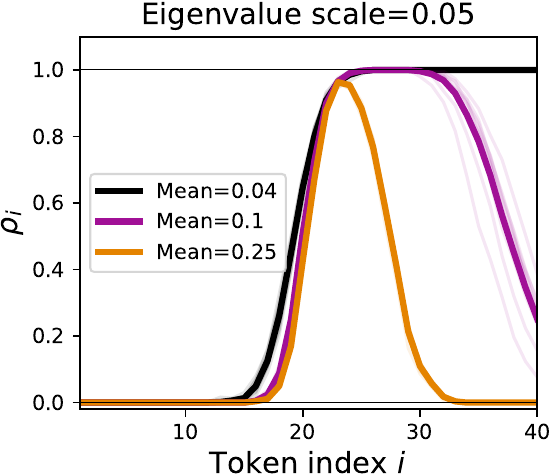} \hfill
    \includegraphics[width=0.24\textwidth]{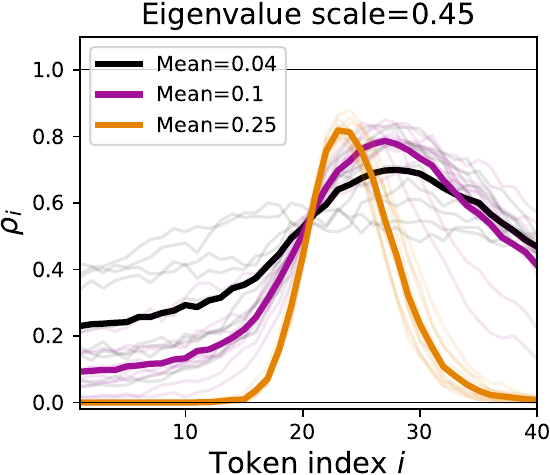} \hfill
    \includegraphics[width=0.24\textwidth]{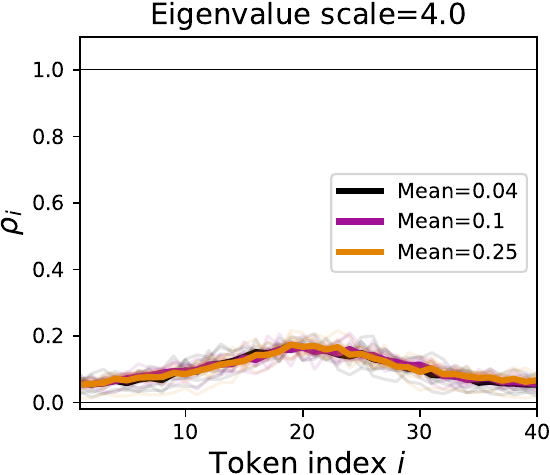} \hfill
    \includegraphics[width=0.24\textwidth]{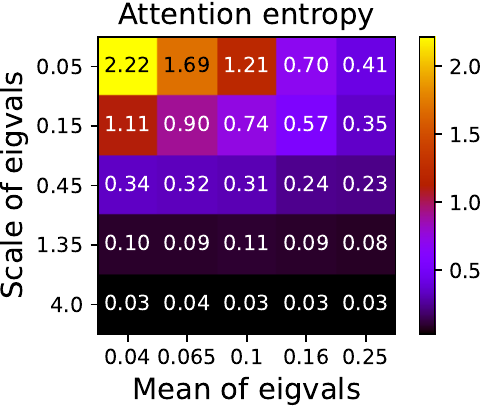}
    \caption{
        Simulated signal propagation probability.
        In the \textbf{top} and \textbf{bottom} rows, the results for the isotropic and anisotropic covariances (the details in the text) are shown, respectively.
        \textbf{(Left)} Signal propagation probability $\rho_i$ computed over repeatedly sampled $300$ random walks (\cref{assumption:random_walk}) with $40$ tokens.
        For each line, $\WQK$ ($d=128$) is sampled $10$ times with the corresponding mean and scale of the eigenvalue distribution, and the averaged $\rho_i$ is denoted by the bold line.
        \textbf{(Right)} The attention entropy \citep{Zhai2023ICML} is computed for $\WQK$ with different eigenvalue mean-scale pairs.
    }
    \label{figure:spp_sample}
\end{figure*}

We discuss the results of our analysis in \cref{section:analysis} and the previous arguments related to attention uniformity.

\paragraph{Connection to rank collapse.}
\citet{Dong2021ICML} showed that self-attention blocks $\Ubf^\ell$ (see \cref{equation:self_attention}) converges to a rank-$1$ matrix $\zbf\onebf^\top$ (for some $\zbf$) with $L \to \infty$ without skip connections or feed-forward blocks, which is called \emph{rank collapse}.%
\footnote{Note that our matrix notation is different from the one used in \citet{Dong2021ICML} so that we chose to let each column of $\Xbf$ store a token, whereas they let each row of $\Xbf$ store a token.}
They argued the importance of avoiding rank collapse for better expressivity because each token embedding in a rank-$1$ self-attention block degenerates to the same.
Hence, the rank collapse is related to the failure mode attributed to the uniformity \emph{after} mixing key tokens by attention, which is slightly different from what we are concerned about---how each token contributes \emph{during} mixing by attention (through the gradient, as discussed in \cref{section:direction}).

\citet[Theorem~2.2]{Dong2021ICML} proved that the convergence rate to a rank-$1$ matrix slows down when the matrix $\ell_1$-norm $\|\WQK\|_1$ is large.
Because we can draw the following connection between $\|\WQK\|_1$ and $|\tr(\Wbf)|$:
\begin{equation}
    \label{equation:l1_norm_lower_bound}
    \frac{|\tr(\Wbf)|}{\sqrt{d}\|\Sigmabf\|_2}
    \le \|\WQK\|_2
    \le \|\WQK\|_\frob
    \le \|\WQK\|_1,
\end{equation}
where the first inequality is due to the bound~\eqref{equation:tr_tr2_inequality} and the Cauchy--Schwarz inequality,
it is advisable to increase $|\tr(\Wbf)|$ under fixed $\tr(\Wbf^2)$ to mitigate the rank collapse.
This is equivalent to reducing the eigenspectrum variance
\begin{equation}
    \label{equation:eigenspectrum_variance}
    d^2\Vbb[w_i] = d^2(\E[w_i^2] - \E[w_i]^2) = d\tr(\Wbf^2) - |\tr(\Wbf)|^2.
\end{equation}
Hence, \emph{minimizing the $\Wbf$-eigenspectrum variance leads to better expressivity}.

\paragraph{Connection to entropy collapse.}
\citet{Zhai2023ICML} introduced a concept called \emph{entropy collapse}, in which the average Shannon entropy of the columns of the attention matrix $\Abf^\ell$ (see \cref{equation:attention_mask}) shrinks.
Intuitively speaking, low attention entropy induces localized attention.
This notion of localization is akin to ours because the attention entropy measures the uniformity the attention is applied to input tokens \emph{during} mixing.
They empirically observed that the training loss falls into a plateau with low attention entropy, which causes training instability of transformers, and hence advocate for keeping attention less peaked during training.

In \citet[Theorem~3.1]{Zhai2023ICML}, the attention entropy is asymptotically lower-bounded for large $T$ by
\begin{equation}
    \label{equation:entropy_lower_bound}
    \ln(1+T\exp(-\nu)) + \frac{\nu\exp(-\nu/2)}{T^{-1} + \exp(-\nu)},
\end{equation}
where $\nu \defeq \|\Xbf\Xbf^\top\|_2\|\WQK\|_2$.
This lower bound is unimodal in $\nu$ and vanishes at $\|\WQK\|_2 \to \infty$ (see \cref{figure:ent_lower_bound}),
so the attention entropy tends to be higher when $\|\WQK\|_2$ is small.
If $|\tr(\Wbf)|$ is not too small, $\|\WQK\|_2$ is lower-bounded (see \cref{equation:l1_norm_lower_bound}) and the attention entropy may be kept close to the peak of the lower bound~\eqref{equation:entropy_lower_bound}.
To mitigate the entropy collapse, it is natural to decrease $\tr(\Wbf^2)$ under fixed $\tr(\Wbf)$ (which is equivalent to minimizing the eigenspectrum variance by \cref{equation:eigenspectrum_variance}) because of the bound
\begin{equation}
    \label{equation:spectral_norm_upper_bound}
    \|\Sigmabf^{-1}\|_\frob\sqrt{\tr(\Wbf^2)}
    \ge |\tr(\WQK)|
    \ge \|\WQK\|_2,
\end{equation}
where the first inequality is due to the Cauchy--Schwarz inequality.
Hence, \emph{minimizing the $\Wbf$-eigenspectrum variance helps the model to avoid the entropy collapse}.

\paragraph{Rank collapse vs. entropy collapse.}
At first sight, the two notions of collapse seem to contradict each other because avoiding the rank collapse leads to diverse token embeddings, whereas avoiding the entropy collapse leads to a uniform token mixer.
Indeed, the matrix $\ell_1$-norm (that decreases under the rank collapse) and the spectral norm (that increases under the entropy collapse) are equivalent norms, and the two modes appear to be incompatible.

However, as we discussed above, this trade-off is reconcilable from the viewpoint of the $\Wbf$-eigenspectrum.
Setting the eigenspectrum mean to be bounded away from zero, we can avoid the rank collapse owing to the bound~\eqref{equation:l1_norm_lower_bound}.
Under a fixed eigenspectrum mean, minimizing the eigenspectrum scale (equivalently, minimizing its variance~\eqref{equation:eigenspectrum_variance}) leads to high attention entropy due to the bound~\eqref{equation:spectral_norm_upper_bound} and the unimodal shape of the entropy lower bound~\eqref{equation:entropy_lower_bound}.
This variance minimization is nothing else but the condition of attention localization.
Eventually, $\rho(\theta)$ localizes and attends to specific sites of tokens, as we showed in \cref{lemma:spp_property}.
Hence, the signal propagation probability offers us a better view of localization.
Let us summarize our second take-home.

\begin{takehome}[frametitle={A2: How does attention localization impact?}]
    \begin{itemize}
        \item \emph{Better expressivity}:
        If $|\tr(\Wbf)|$ is maximized for a fixed $\tr(\Wbf^2)$, the convergence to the rank collapse becomes slow.

        \item \emph{High attention entropy}:
        If $\tr(\Wbf^2)$ is minimized for a fixed $|\tr(\Wbf)|$ bounded away from zero, the attention entropy is increased.
    \end{itemize}
    Both of the above are attributed to minimizing the $\Wbf$-eigenspectrum variance~\eqref{equation:eigenspectrum_variance}.
\end{takehome}

\paragraph{Numerical simulation.}
To see the relationship between the eigenspectrum, $\rho$, and attention entropy, we simulated the signal propagation probability $\rho_i$ using synthesized random walks following \cref{assumption:random_walk} with the isotropic $\Sigmabf = \Ibf$ and anisotropic $\Sigmabf$.
To obtain an anisotropic $\Sigmabf$, we first sampled $\Rbf \in \Rbb^{d \times d}$ from element-wise $\mathrm{Unif}(-2.5,2.5)$ and computed $\Sigmabf = \Rbf^\top\Rbf/d$.
We sampled $300$ sequences with $40$ tokens, and obtained $\WQK$ by generating $128$ eigenvalues $(w_i)_{i \in [d]}$ from $\Ncal(\text{mean}, \text{scale}^2)$ and composed with a sampled orthogonal basis matrix $\Bbf$, by the eigendecomposition formula $\WQK = \Bbf\diag((w_i)_{i})\Bbf^\top$.
The signal propagation probability was averaged over $300$ sequences.

\Cref{figure:spp_sample} shows the averaged $\rho_i$ with different mean-scale pairs and the corresponding attention entropy in the right-most figure.
As seen, $\rho_i$ localizes with smaller scales and larger means, which is consistent with the conclusion in \cref{section:analysis}.
This trend supports the validity of $\WQK$-eigenspectrum as a proxy to $\Wbf$-eigenspectrum.
Moreover, we observe that $\WQK$-eigenspectrum with a fixed mean and scale leads to higher attention entropy.

\begin{figure*}
    \centering
    \includegraphics[width=0.245\textwidth]{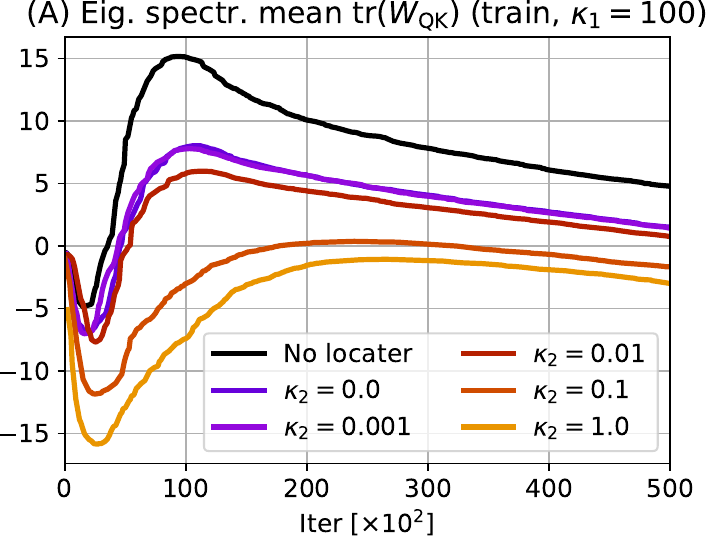} \hfill
    \includegraphics[width=0.245\textwidth]{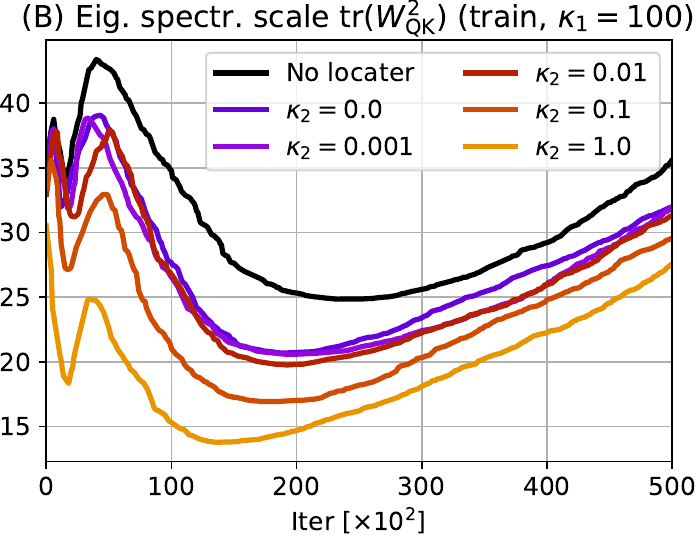} \hfill
    \includegraphics[width=0.245\textwidth]{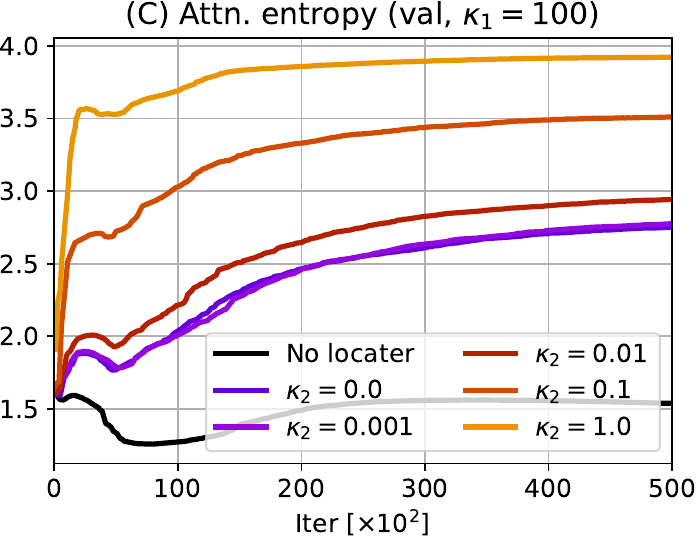} \hfill
    \includegraphics[width=0.245\textwidth]{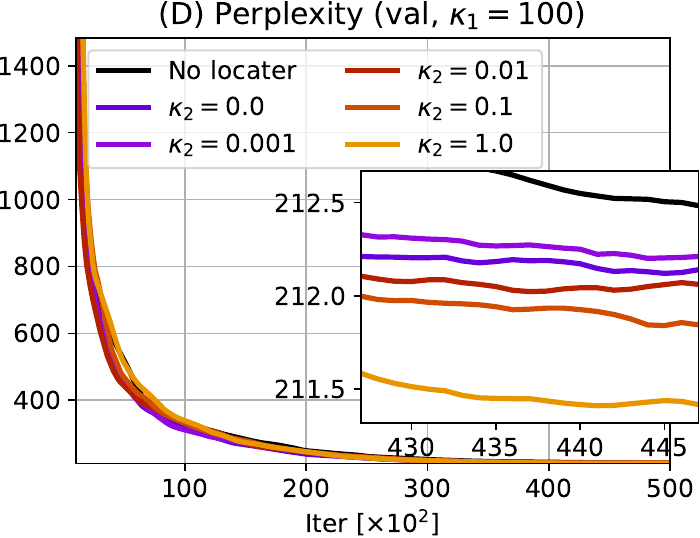}
    \caption{
        Experimental results of language modeling (WikiText-2) with $d=128$ and $1$-layer transformers, fixed $\kappa_1=100$, and varying regularization intensity $\kappa_2$.
        With stronger $\kappa_2$, the eigenspectrum scale shrinks \textbf{(B)}, the attention entropy increases \textbf{(C)}, and the perplexity improves \textbf{(D)}.
        The rightmost figure magnifies the x-range, where the perplexity attains the minimum.
    }
    \label{figure:result_lm_d128_l1}
\end{figure*}

\begin{figure}
    \centering
    \includegraphics[width=0.5\textwidth]{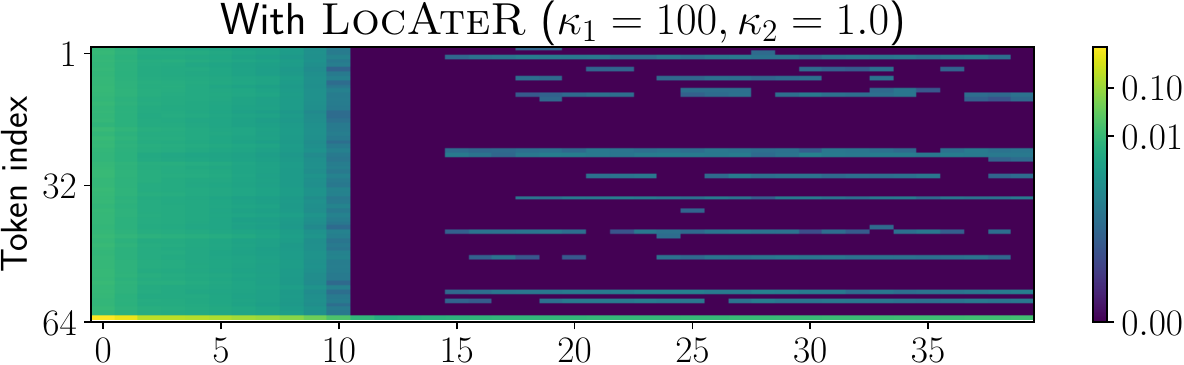}
    \includegraphics[width=0.5\textwidth]{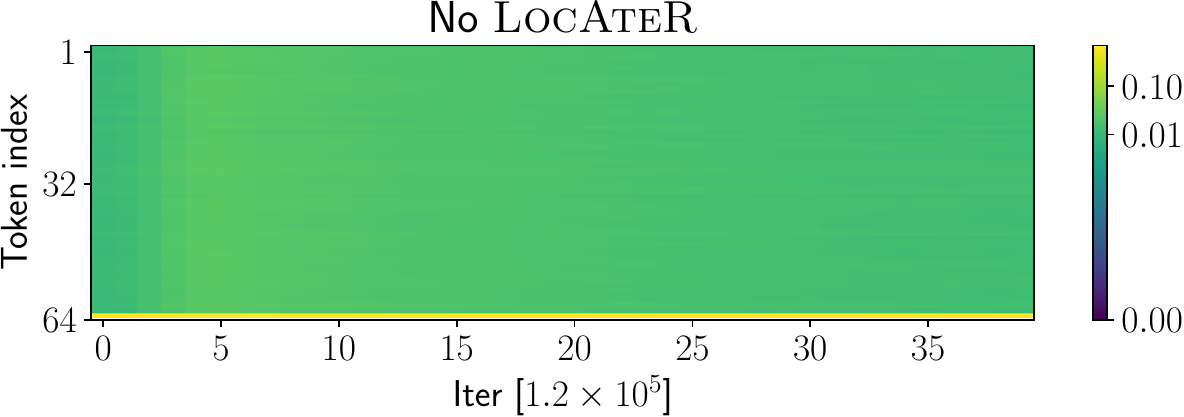}
    \caption{
        The signal propagation probabilities are shown at each iteration over $50000$ iterations.
        \textbf{(Top)} \textsc{LocAteR} with $\kappa_1=100$ and $\kappa_2=1$. A couple of light and dark horizontal stripes correspond to the attention localization.
        \textbf{(Bottom)} No \textsc{LocAteR}. Overall, the signal propagation probability is uniform at each time.
    }
    \label{figure:spp_evolution}
\end{figure}

\section{Intervening attention localization}
\label{section:proposed}

To empirically see the impact of localization on the model performance, we propose a method to control the degree of attention localization.
We focus on the eigenspectrum of $\WQK$ instead of $\Wbf = \WQK\Sigmabf$ because $\Sigmabf$ does not change during training, and the numerical simulation showed that $\WQK$ behaves as a reasonable proxy to $\Wbf$ (see \cref{section:related}).

We minimize the loss function $J$ while minimizing the eigenspectrum scale and maintaining the mean to a fixed level:
\[
    \min_{\Thetabf} \Big\{ J(\Thetabf) + \kappa_1\tr(\WQK^\top\WQK) + \kappa_2(\tr(\WQK) - 1)^2 \Big\},
\]
where $\kappa_1, \kappa_2 > 0$ are the regularization strengths.
Here, we allow $\WQK$ to be asymmetric, and the eigenspectrum scale is represented by $\tr(\WQK^\top\WQK)$.
The regularization terms can be optimized fairly easily thanks to the following derivative formulae (for $\WQK = \WQ^\top\WK$):
\[
    \begin{aligned}
        &\nabla_{\WQ}\tr(\WQK^\top\WQK) = 2\WK\WK^\top\WQ, \\
        &\nabla_{\WK}\tr(\WQK^\top\WQK) = 2\WQ\WQ^\top\WK, \\
        &\nabla_{\WQ}[(\tr(\WQK)-1)^2] \!=\! [\tr(\WQK)-1]\tr(\WQK)\WK, \\
        &\nabla_{\WK}[(\tr(\WQK)-1)^2] \!=\! [\tr(\WQK)-1]\tr(\WQK)\WQ.
    \end{aligned}
\]
Since this whole objective drives the eigenspectrum scale to a small value, the signal vanishing can be avoided automatically.
We call this regularization scheme \textsc{LocAteR} (LOCalized ATtEntion Regularization).

\section{Experiments}
\label{section:experiments}

We aim to observe the correlation between the eigenspectrum and localization.
To this end, we train transformers with \textsc{LocAteR} and varying $\kappa_1,\kappa_2$, and see how the model performances and attention foci change over time.

\paragraph{Setup.}
We used fairseq \verb|v0.12.2|~\citep{Ott2019NACCL}, which is a toolkit oriented for sequence modeling, to implement and train transformers.
The basic training scheme was inherited from \href{https://github.com/facebookresearch/fairseq/blob/main/fairseq_cli/train.py}{\texttt{fairseq-cli/train.py}}.
The model is a $1$-layer transformer with a single-head self-attention and Post-LN (default), and the input embedding dimension, attention embedding dimension, and feed-forward net embedding dimension are set to $128$ altogether (namely, $d=128$).%
\footnote{
    The experimental results with deeper transformers are shown in \cref{section:additional_experiments}.
    The overall trends remain alike.
}
Input data were transformed into $64$ tokens (namely, $T=64$) with batch size $64$.
The optimizer is Adam~\citep{Kingma2015ICLR} with default parameters and no clip norm, and the weight decay with $0.01$ is used.
The learning rate is fixed to $2.5\times10^{-5}$ without any scheduling.
The FP16 quantizer was applied to reduce memory usage.
All the other configs remain to be the same as the default in \texttt{fairseq-cli/train.py}.
Under this config, we updated the model with $50000$ iters.

\paragraph{Language modeling.}
We conducted the language modeling task.
The dataset we used is WikiText-2~\citep{Merity2016}, which is a collection of high-quality Wikipedia articles.
We conduced the experiments with fixed scale regularization strength $\kappa_1=100$ and varying mean regularization strength $\kappa_2$ from $0, 10^{-3}, 10^{-2}, \dots, 10^{0}$.

The results are shown in \cref{figure:result_lm_d128_l1}, in which stronger regularizers tend to make the eigenspectrum scale smaller.
This, in turn, maintains the attention entropy higher during the updates entirely, and eventually, the model achieves better perplexity.
While the better model performance with higher attention entropy has already been observed by \citet{Zhai2023ICML}, we also showed that a smaller eigenspectrum scale contributes to higher attention entropy.
This empirically corroborates that attention localization leads to better model performance, probably because the attention mechanism appropriately selects relevant tokens during training.

\Cref{figure:spp_evolution} shows the signal propagation probability at each training iteration.
We compute the signal propagation probability of token $i$ by counting the frequency of $\inpr{\gammabf^i}{\omegabf} + \gamma^i_0 \in [0,1]$ in a given batch.
\textsc{LocAteR} entails salient horizontal stripes, each corresponding to attended tokens.
Yet, the stripes do not appear in ``bulk'' as we analyzed in \cref{section:analysis} because our synthetic data model in \cref{assumption:random_walk} does not perfectly align with real datasets.
Nevertheless, our experiments evidently contrast the localized and uniformed attention depending on the eigenspectrum scale because no salient stripes are observed without \textsc{LocAteR}.

In \cref{figure:result_lm_d128_l1,figure:spp_evolution}, we observe different learning phases for the first $10^4$ and the rest iters.
Indeed, \citet{Tian2023} observed similar phenomena and explained that it is due to the different convergence speeds between attention weights corresponding to informative and non-informative tokens.
The relationship between the $\WQK$-eigenspectrum and this dynamics is beyond our scope and left for future work.

\section{Conclusion and limitation}
\label{section:conclusion}

We revealed that attention localizes when the eigenspectrum of $\Wbf$ concentrates to a non-zero mean, or equivalently, with larger eigenspectrum mean $\tr(\Wbf)$ and smaller scale $\tr(\Wbf^2)$.
Based on it, \textsc{LocAteR} was proposed to shrink the scale $\tr(\WQK^2)$ while maintaining the mean $\tr(\WQK)$.
Interestingly, maximizing the scale is related to mitigating both rank collapse and entropy collapse, and hence, the two apparently contradictory failure modes can be reconciled.
The experiments on a real-world dataset corroborate it, though the random walk model is not perfectly satisfied.

We recognize three limitations of this work.
First, we rely on the strong random walk model.
Although the Gaussianity may be reasonable because of usual initialization schemes of transformer embedding layers, it is interesting to consider an alternative model to capture token correlations better.
Second, the formal analysis is mainly restricted to $1$-layer transformers.
Recent studies often consider gradient explosion in the large-depth limit from the viewpoint of layer normalization~\citep{Xiong2020ICML,Takase2023ACL} and initialization~\citep{Bachlechner2021UAI,Takase2023}.
It must be fruitful to integrate these perspectives to our gradient analysis through \cref{equation:qk_grad}.
Lastly, why attention localization leads to better model performance still remains elusive.
Whereas localization is related to avoiding rank collapse (and hence higher model expressivity), we need additional effort to fully understand the mechanism.

\section*{Acknowledgments}
A part of the experiments of this research was conducted using Wisteria/Aquarius in the Information Technology Center, the University of Tokyo.

\bibliography{reference}
\bibliographystyle{icml2024}

\newpage
\appendix
\onecolumn

\setlength{\abovedisplayskip}{11pt}
\setlength{\belowdisplayskip}{11pt}

\section{Helper lemmas}
\label{section:helper}

\begin{lemma}
    \label{lemma:moments}
    Let $\Wbf \in \Rbb^{d \times d}$ be a symmetric matrix.
    Fix $\abf, \mubf \in \Rbb^d$ and $\Sigmabf \in \Rbb^{d \times d}$ be a covariance matrix.
    For $\xbf \sim \Ncal(\mbf, \Sigmabf)$, the following moment formulae hold:
    \begin{align}
        \label{equation:moment_quad}
        \E[\xbf^\top\Wbf\xbf] &= \tr(\Wbf\Sigmabf) + \mbf^\top\Wbf\mbf \\
        \label{equation:moment_outer}
        \E[\xbf\xbf^\top] &= \Sigmabf + \mbf\mbf^\top \\
        \label{equation:moment_cubic}
        \E[\abf^\top\Wbf\xbf\xbf^\top\Wbf\xbf] &= 2\abf^\top\Wbf\Sigmabf\Wbf\mbf + \abf^\top\Wbf\mbf\{\tr(\Wbf\Sigmabf) + \mbf^\top\Wbf\mbf\} \\
        \label{equation:moment_quad_quad}
        \E[\xbf^\top\Wbf\xbf\xbf^\top\Wbf\xbf] &= 2\tr(\Wbf\Sigmabf\Wbf\Sigmabf) + \tr(\Wbf\Sigmabf)^2 + 4\mbf^\top\Wbf\Sigmabf\Wbf\mbf + 2\tr(\Wbf\Sigmabf)\mbf^\top\Wbf\mbf + \mbf^\top\Wbf\mbf\mbf^\top\Wbf\mbf
    \end{align}
\end{lemma}

The formulae in \cref{lemma:moments} are standard and cropped from \citet{Brookes1998}.

\begin{lemma}
    \label{lemma:rec_moments}
    Let $\Wbf \in \Rbb^{d \times d}$ be a symmetric matrix.
    For $i \le j$, suppose that $\xbf_i, \xbf_j$ follow \cref{assumption:random_walk}.
    Then, the following formulae hold:
    \begin{align}
        \label{equation:rec_moment_quad}
        \E[\xbf_i^\top\Wbf\xbf_i] &= (i-1)\tr(\Wbf\Sigmabf) \\
        \label{equation:rec_moment_quad_quad:1}
        \E[\xbf_i^\top\Wbf\xbf_i\xbf_i^\top\Wbf\xbf_i] &= (i^2-2i+2)\{2\tr(\Wbf\Sigmabf\Wbf\Sigmabf) + \tr(\Wbf\Sigmabf)^2\} \\
        \label{equation:rec_moment_quad_quad:2}
        \E[\xbf_i^\top\Wbf\xbf_j\xbf_j^\top\Wbf\xbf_i] &= (i^2+ij-3i-j+4)\tr(\Wbf\Sigmabf\Wbf\Sigmabf) + (i^2-2i+2)\tr(\Wbf\Sigmabf)^2 \\
        \label{equation:rec_moment_cubic}
        \E[\xbf_i^\top\Wbf\xbf_j\xbf_j^\top\Wbf\xbf_j] &= (ij-i-j+2)\{2\tr(\Wbf\Sigmabf\Wbf\Sigmabf) + \tr(\Wbf\Sigmabf)^2\}
    \end{align}
\end{lemma}

The formulae in \cref{lemma:rec_moments} can be shown by recursively applying \cref{lemma:moments}.
We omit the proofs since they are elementary.

\section{Omitted derivations}
\label{section:derivation}

\subsection{QK-gradient}
Here, we complement the derivations of the QK-gradient terms shown in \cref{section:direction}.
To get \cref{equation:qk_grad}, we compute $\nabla_{\WQK}J$:
\[
\begin{aligned}
    \nabla_{\WQK}J
    &= \frac{1}{2}\E[\nabla_{\WQK}\|\ybf - \Fbf(\Xbf)_T\|^2] \\
    &= \frac{1}{2}\E[\nabla_{\WQK}\{\Fbf(\Xbf)_T^\top\Fbf(\Xbf)_T - 2\ybf^\top\Fbf(\Xbf)_T\}] \\
    &= \frac{1}{2}\E[\omegabf^\top\Gammabf\Xbf^\top\WV^\top\WF^\top\WF\WV\Xbf\Gammabf^\top\omegabf + 2(\WV\Xbf\tilde\gammabf_0 + \xbf_T)^\top\WF^\top\WF\WV\Xbf\Gammabf^\top\omegabf - 2\ybf^\top\WF\WV\Xbf\Gammabf^\top\omegabf] \\
    &= \frac{1}{2\lambda^2}\E[\xbf_T^\top\WQK^\top\Xbf\Gammabf\Pbf\Gammabf^\top\Xbf^\top\WQK\xbf_T] + \frac{1}{\lambda}\E[(\tilde\gammabf_0)^\top\Pbf\Gammabf^\top\Xbf^\top\WQK\xbf_T] + \frac{1}{\lambda}\E[\qbf^\top\Gammabf^\top\Xbf^\top\WQK\xbf_T] \\
    &= \frac{1}{\lambda^2}\E[\Xbf\Gammabf\Pbf\Gammabf^\top\Xbf^\top\WQK\xbf_T\xbf_T^\top] + \frac{1}{\lambda}\E[\Xbf\Gammabf\Pbf\tilde\gammabf_0\xbf_T^\top] + \frac{1}{\lambda}\E[\Xbf\Gammabf\qbf\xbf_T^\top].
\end{aligned}
\]
By expanding the first term of $\nabla_{\WQK}J$, we get the following:
\begin{equation}
    \label{equation:qk_grad_1_}
    \frac{1}{\lambda^2} \!\! \sum_{i,j,\alpha,\beta \in [T]} \!\! \E\left[ \tilde\gamma^i_\alpha\tilde\gamma^j_\beta (\xbf_i^\top\check\Pbf\xbf_j)(\xbf_\beta^\top\WQK\xbf_T)\xbf_\alpha\xbf_T^\top \right],
\end{equation}
where $\check\Pbf \defeq \WV^\top\WF^\top\WF\WV$.
Similarly, by expanding the second and third terms of $\nabla_{\WQK}J$, we get the following terms, respectively:
\begin{align}
    \label{equation:qk_grad_2}
    &\frac{1}{\lambda} \sum_{i,\alpha,\beta \in [T]} \E\left[\tilde\gamma^i_\alpha\tilde\gamma_0^\beta (\xbf_i^\top\check\Pbf\xbf_\beta)\xbf_\alpha\xbf_T^\top\right], \\
    \label{equation:qk_grad_3}
    &\frac{1}{\lambda} \sum_{i,\alpha \in [T]} \E\left[\tilde\gamma^i_\alpha \{\xbf_i^\top\WV^\top\WF^\top(\WF\xbf_T - \ybf)\}\xbf_\alpha\xbf_T^\top\right].
\end{align}

To get \cref{equation:qk_grad_1_}, we expand the first term of $\nabla_{\WQK}J$:
\[
\begin{aligned}
  \E[\Xbf\Gammabf\Pbf\Gammabf^\top\Xbf^\top\WQK\xbf_T\xbf_T^\top]  
  &= \E[\Xbf\Gammabf\Xbf^\top\WV^\top\WF^\top\WF\WV\Xbf\Gammabf^\top\Xbf^\top\WQK\xbf_T\xbf_T^\top] \\
  &= \E\left[\{(\WF\WV\Xbf)(\Xbf\Gammabf)^\top\}^\top\{(\WF\WV\Xbf)(\Xbf\Gammabf)^\top\}\WQK\xbf_T\xbf_T^\top\right] \\
  &= \E\bigg[\bigg\{\sum_{i \in [T]}(\Xbf\tilde\gammabf^i)(\WF\WV\xbf_i)^\top\bigg\} \bigg\{\sum_{j \in [T]}(\WF\WV\xbf_j)(\Xbf\tilde\gammabf^j)^\top\bigg\}\WQK\xbf_T\xbf_T^\top\bigg] \\
  &= \E\bigg[\sum_{i,j\in[T]} (\Xbf\tilde\gammabf^i)\{(\WF\WV\xbf_i)^\top(\WF\WV\xbf_j)\}(\Xbf\tilde\gammabf^j)^\top\WQK\xbf_T\xbf_T^\top\bigg] \\
  &= \E\bigg[\sum_{\alpha,\beta\in[T]}\bigg\{\sum_{i,j\in[T]}\xbf_i^\top\check\Pbf\xbf_j\bigg\}\tilde\gamma^i_\alpha\tilde\gamma^j_\beta\xbf_\alpha\xbf_\beta^\top\WQK\xbf_T\xbf_T^\top\bigg] \\
  &= \sum_{i,j,\alpha,\beta\in[T]} \E[\tilde\gamma^i_\alpha\tilde\gamma^j_\beta(\xbf_i^\top\check\Pbf\xbf_j)(\xbf_\beta^\top\WQK\xbf_T)\xbf_\alpha\xbf_T^\top].
\end{aligned}
\]
To get \cref{equation:qk_grad_2}, we expand the second term of $\nabla_{\WQK}J$:
\[
\begin{aligned}
    \E[\Xbf\Gammabf\Pbf\tilde\gammabf_0\xbf_T^\top]
    &= \E[\Xbf\Gammabf\Xbf^\top\WV^\top\WF^\top\WF\WV\Xbf\tilde\gammabf_0\xbf_T^\top] \\
    &= \E[\{(\WF\WV\Xbf)(\Xbf\Gammabf)^\top\}^\top\{(\WF\WV\Xbf)(\tilde\gammabf_0\xbf_T^\top)\}] \\
    &= \E\bigg[\bigg\{\sum_{i\in[T]}(\Xbf\tilde\gammabf^i)(\WF\WV\xbf_i)^\top\bigg\} \bigg\{\sum_{\beta\in[T]}(\WF\WV\xbf_\beta)(\tilde\gamma_0^\beta\xbf_T^\top)\bigg\}\bigg] \\
    &= \E\bigg[\sum_{i,\beta\in[T]}(\Xbf\tilde\gammabf^i)\{(\WF\WV\xbf_i)^\top(\WF\WV\xbf_\beta)\}(\tilde\gamma_0^\beta\xbf_T^\top)\bigg] \\
    &= \E\bigg[\sum_{\alpha\in[T]}\bigg\{\sum_{i,\beta\in[T]}\xbf_i^\top\check\Pbf\xbf_\beta\bigg\}\tilde\gamma^i_\alpha\tilde\gamma_0^\beta\xbf_\alpha\xbf_T^\top\bigg] \\
    &= \sum_{i,\alpha,\beta\in[T]} \E[\tilde\gamma^i_\alpha\tilde\gamma_0^\beta(\xbf_i^\top\check\Pbf\xbf_\beta)\xbf_\alpha\xbf_T^\top].
\end{aligned}
\]
To get \cref{equation:qk_grad_3}, we expand the third term of $\nabla_{\WQK}J$:
\[
\begin{aligned}
    \E[\Xbf\Gammabf\qbf\xbf_T^\top]
    &= \E[\Xbf\Gammabf\Xbf^\top\WV^\top\WF^\top(\WF\xbf_T - \ybf)\xbf_T^\top] \\
    &= \E[\{(\WF\WV\Xbf)(\Xbf\Gammabf)^\top\}^\top(\WF\xbf_T - \ybf)\xbf_T^\top] \\
    &= \E\bigg[\bigg\{\sum_{i\in[T]}(\Xbf\tilde\gammabf^i)(\WF\WV\xbf_i)^\top\bigg\}(\WF\xbf_T - \ybf)\xbf_T^\top\bigg] \\
    &= \E\bigg[\bigg\{\sum_{i,\alpha\in[T]}\tilde\gamma^i_\alpha\xbf_\alpha\xbf_i^\top\WV^\top\WF^\top\bigg\}(\WF\xbf_T - \ybf)\xbf_T^\top\bigg] \\
    &= \sum_{i,\alpha\in[T]} \E[\tilde\gamma^i_\alpha\{\xbf_i^\top\WV^\top\WF^\top(\WF\xbf_T - \ybf)\}\xbf_\alpha\xbf_T^\top].
\end{aligned}
\]

\subsection{Order evaluation of QK-gradient}
The orders of the QK-gradient terms \eqref{equation:qk_grad_1_}, \eqref{equation:qk_grad_2}, and \eqref{equation:qk_grad_3} are evaluated.
In this subsection, we assume that the covariance matrix in \cref{assumption:random_walk} is $\Sigmabf = \Ibf$ for simplicity.
The following evaluation still applies with minor modifications for a general $\Sigmabf$.
For \cref{equation:qk_grad_1_}, the Cauchy--Schwarz inequality implies that
\[
\begin{aligned}
    \text{\cref{equation:qk_grad_1_}}
    &= \bigg|\sum_{i,j,\alpha,\beta} \E[\tilde\gamma^i_\alpha\tilde\gamma^j_\beta(\xbf_i^\top\check\Pbf\xbf_j)(\xbf_\beta^\top\WQK\xbf_T)\xbf_\alpha\xbf_T^\top]\bigg|^2 \\
    &\le \underbrace{\bigg\{\sum_{i,j,\alpha,\beta} \E[(\tilde\gamma^i_\alpha\tilde\gamma^j_\beta)^2]\bigg\}}_\text{(A)} \underbrace{\bigg\{\sum_{i,j,\alpha,\beta} \E[(\xbf_i^\top\check\Pbf\xbf_j)^2]\bigg\}}_\text{(B)} \underbrace{\bigg\{\sum_{i,j,\alpha,\beta} \E[(\xbf_\beta^\top\WQK\xbf_T)^2]\bigg\}}_\text{(C)} \underbrace{\bigg\{\sum_{i,j,\alpha,\beta} \E[(\xbf_\alpha\xbf_T^\top)^{\odot2}]\bigg\}}_\text{(D)}.
\end{aligned}
\]
For (A), we have $\tilde\gamma^i_\alpha, \tilde\gamma^j_\beta \le T^{-1}$ by definition, and hence $\text{(A)} = O(1)$.
For (B), by using \cref{equation:rec_moment_quad_quad:2},
\[
\begin{aligned}
    \text{(B)}
    = T^2\sum_{i,j} \E[\xbf_i^\top\check\Pbf\xbf_j\xbf_j^\top\check\Pbf\xbf_i]
    = T^2\sum_{i,j} \{(i^2+ij-3i-j+4)\tr(\check\Pbf^2) + (i^2-2i+2)\tr(\check\Pbf)^2\}
    = O(T^6).
\end{aligned}
\]
For (C), by following the same computation as \cref{equation:rec_moment_quad_quad:sub},
\[
\begin{aligned}
    \text{(C)}
    &= T^3\sum_\beta \E[\xbf_\beta^\top\WQK\xbf_T\xbf_T^\top\WQK\xbf_\beta] \\
    &= T^3\sum_\beta \{(\beta^2+(T-3)\beta-(T-4))\tr(\WQK^2) + (\beta^2-2\beta+2)\tr(\WQK)^2\}
    = O(T^6).
\end{aligned}
\]
For (D), its $(i,j)$-element can be evaluated as follows (no matter whether $i=j$ or not):
\[
\begin{aligned}
    \text{(D)}_{ij}
    = T^3\sum_\alpha \E[x_{\alpha,i}^2x_{T,j}^2]
    = T^3\sum_\alpha \E[x_{\alpha,i}^2\{(T-\alpha) + x_{\alpha,j}^2\}]
    = O(T^5)\sum_\alpha \E[x_{\alpha,i}^2] + T^3\sum_{\alpha} \E[x_{\alpha,i}^2x_{\alpha,j}^2]
    = O(T^6).
\end{aligned}
\]
By plugging them back, we now confirmed that $|\text{\cref{equation:qk_grad_1_}}| = O(T^8)$.

The orders of \cref{equation:qk_grad_2,equation:qk_grad_3} can be evaluated similarly and the detailed evaluations are omitted.
\[
\begin{aligned}
    |\text{\cref{equation:qk_grad_2}}|^2
    &= \bigg|\sum_{i,\alpha,\beta} \E[\tilde\gamma^i_\alpha\tilde\gamma_0^\beta(\xbf_i^\top\check\Pbf\xbf_\beta)\xbf_\alpha\xbf_T^\top]\bigg|^2 \\
    &\le \bigg\{\sum_{i,\alpha,\beta} \E[(\tilde\gamma^i_\alpha\tilde\gamma_0^\beta)^2]\bigg\} \bigg\{\sum_{i,\alpha,\beta} \E[(\xbf_i^\top\check\Pbf\xbf_\beta)^2]\bigg\} \bigg\{\sum_{i,\alpha,\beta} \E[(\xbf_\alpha\xbf_T^\top)^{\odot2}]\bigg\} \\
    &= O(T) \cdot O(T^4) \cdot O(T^5) \\
    &= O(T^{10}) \\
    \implies
    |\text{\cref{equation:qk_grad_2}}|
    &= O(T^5). \\
    |\text{\cref{equation:qk_grad_3}}|^2
    &= \bigg|\sum_{i,\alpha} \E[\tilde\gamma^i_\alpha\{\xbf_i^\top\WV^\top\WF^\top(\WF\xbf_T-\ybf)\}\xbf_\alpha\xbf_T^\top]\bigg|^2 \\
    &\le \bigg\{\sum_{i,\alpha} \E[(\tilde\gamma^i_\alpha)^2]\bigg\} \bigg\{\sum_{i,\alpha} \E[(\xbf_i^\top\WV^\top\WF^\top(\WF\xbf_T-\ybf))^2]\bigg\} \bigg\{\sum_{i,\alpha} \E[(\xbf_\alpha\xbf_T^\top)^{\odot2}]\bigg\} \\
    &= O(1) \cdot O(T^4) \cdot O(T^4) \\
    &= O(T^8) \\
    \implies
    |\text{\cref{equation:qk_grad_3}}|
    &= O(T^4).
\end{aligned}
\]

Hence, we have $|\text{\cref{equation:qk_grad_1_}}| = O(T^8)$, $|\text{\cref{equation:qk_grad_2}}| = O(T^5)$, and $|\text{\cref{equation:qk_grad_3}}| = O(T^4)$, which implies that the QK-gradient~\eqref{equation:qk_grad} is asymptotically dominated by \cref{equation:qk_grad_1_}.

\section{Proofs}
\label{section:proofs}

\sppmeanvar*

\begin{proof}
    To derive the mean, we use \cref{equation:rec_moment_quad}.
    \begin{equation*}
        \begin{aligned}
            \mu^i
            &= \frac{1}{\lambda T}\E[\xbf_i^\top\WQK\xbf_T] - \frac{1}{\lambda T^2}\sum_{j \in [T]}\E[\xbf_j^\top\WQK\xbf_T] + o(1) \\
            &= \frac{i-1}{\lambda T}\tr(\Wbf) - \frac{\sum_{j \in [T]} (j-1)}{\lambda T^2}\tr(\Wbf) + o(1) \\
            &= \left(i - \frac{T+1}{2}\right)\frac{\tr(\Wbf)}{\lambda T} + o(1).
        \end{aligned}
    \end{equation*}
    Note that $\gamma_0^i = o(1)$.

    To derive the variance, we first compute $\E[\xbf_i^\top\WQK\xbf_T\xbf_j^\top\WQK\xbf_T]$ (for $i \le j \le T$).
    \begin{equation}
        \label{equation:rec_moment_quad_quad:sub}
        \begin{aligned}
            \E[\xbf_i^\top\WQK\xbf_T\xbf_j^\top\WQK\xbf_T]
            &= \E[\xbf_i^\top\WQK(\xbf_T\xbf_T^\top)\WQK\xbf_j] \\
            &= \E[\xbf_i^\top\WQK\{(T-j)\Sigmabf + \xbf_j\xbf_j^\top\}\WQK\xbf_j] \\
            &= (T-j)\E[\xbf_i^\top\WQK\Sigmabf\WQK\xbf_j] + \E[\xbf_i^\top\WQK\xbf_j\xbf_j^\top\WQK\xbf_j] \\
            &= (T-j)(i-1)\tr(\Wbf^2) + (ij-i-j+2)\{2\tr(\Wbf^2) + \tr(\Wbf)^2\} \\
            &= (ij+(T-2)i-j-(T-4))\tr(\Wbf^2) + (ij-i-j+2)\tr(\Wbf)^2,
        \end{aligned}
    \end{equation}
    where \cref{equation:moment_quad} is used recursively at the second identity and \cref{equation:rec_moment_quad,equation:rec_moment_cubic} are used at the fourth identity.
    Then, the expectation of the squared term is expanded:
    \begin{equation*}
        \begin{aligned}
            \E&[\inpr{\gammabf^i}{\Xbf^\top\WQK\xbf_T}^2] \\
            &= \E\bigg[\bigg(\frac{1}{T}\xbf_i^\top\WQK\xbf_T - \frac{1}{T^2}\sum_{j \in [T]}\xbf_j^\top\WQK\xbf_T\bigg)^2\bigg] \\
            &= \E\bigg[\frac{1}{T^2}\xbf_i^\top\WQK\xbf_T\xbf_i^\top\WQK\xbf_T - \frac{2}{T^3}\sum_{j \in [T]}\xbf_i^\top\WQK\xbf_T\xbf_j^\top\WQK\xbf_T + \frac{1}{T^4}\sum_{j,j' \in [T]}\xbf_j^\top\WQK\xbf_T\xbf_{j'}^\top\WQK\xbf_T\bigg] \\
            &= \frac{1}{T^2}\underbrace{\E[\xbf_i^\top\WQK\xbf_T\xbf_i^\top\WQK\xbf_T]}_\text{(A)} \\
                &\phantom{=} - \frac{1}{T^3}\underbrace{2\E[\xbf_i^\top\WQK\xbf_T\xbf_i^\top\WQK\xbf_T]}_\text{(B1)} - \frac{1}{T^3}\underbrace{2\sum_{j>i}\E[\xbf_i^\top\WQK\xbf_T\xbf_j^\top\WQK\xbf_T]}_\text{(B2)} - \frac{1}{T^3}\underbrace{2\sum_{j<i}\E[\xbf_i^\top\WQK\xbf_T\xbf_j^\top\WQK\xbf_T]}_\text{(B3)} \\
                &\phantom{=} + \frac{1}{T^4}\underbrace{\sum_{j\in[T]}\E[\xbf_j^\top\WQK\xbf_T\xbf_j^\top\WQK\xbf_T]}_\text{(C1)} + \frac{1}{T^4}\underbrace{2\sum_{j<j'}\E[\xbf_j^\top\WQK\xbf_T\xbf_{j'}^\top\WQK\xbf_T]}_\text{(C2)}.
        \end{aligned}
    \end{equation*}
    Each term is computed by using \cref{equation:rec_moment_quad_quad:sub} multiple times.
    \begin{equation*}
        \begin{aligned}
            \text{(A)}
            &= (i^2+(T-3)i-(T-4))\tr(\Wbf^2) + (i^2-2i+2)\tr(\Wbf)^2 \\
            &= (i^2+Ti)\tr(\Wbf^2) + i^2\tr(\Wbf)^2 + o(T^2), \\
            \text{(B1)}
            &= o(T^3), \\
            \text{(B2)}
            &= 2\sum_{j>i} \{(ij+(T-2)i-j-(T-4))\tr(\Wbf^2) + (ij-i-j+2)\tr(\Wbf)^2\} \\
            &= (T^2i-2Ti^2-i^3)\tr(\Wbf^2) + (T^2i-i^3)\tr(\Wbf)^2 + o(T^3), \\
            \text{(B3)}
            &= 2\sum_{j<i} \{(ij+(T-2)j-i-(T-4))\tr(\Wbf^2) + (ij-i-j+2)\tr(\Wbf)^2\} \\
            &= (Ti^2+i^3)\tr(\Wbf^2) + i^3\tr(\Wbf)^2 + o(T^3), \\
            \text{(C1)}
            &= \sum_{j\in[T]} \{(j^2+(T-3)j-(T-4))\tr(\Wbf^2) + (j^2-2j+2)\tr(\Wbf)^2\} \\
            &= o(T^4), \\
            \text{(C2)}
            &= 2\sum_{j<j'} \{(jj'+(T-2)j-j'-(T-4))\tr(\Wbf^2) + (jj'-j-j'+2)\tr(\Wbf)^2\} \\
            &= 2\sum_{j<j'} jj'\{\tr(\Wbf^2) + \tr(\Wbf)^2\} + 2\sum_{j<j'}Tj\tr(\Wbf^2) + o(T^4) \\
            &= \sum_{j\in[T]} j(T-j)(T+j+1) \{\tr(\Wbf^2) + \tr(\Wbf)^2\} + 2T\sum_{j\in[T]}(T-j)j\tr(\Wbf^2) + o(T^4) \\
            &= \sum_{j\in[T]} (T^2j-j^3) \{\tr(\Wbf^2) + \tr(\Wbf)^2\} + \frac{T^4}{3}\tr(\Wbf^2) + o(T^4) \\
            &= \frac{7T^4}{12}\tr(\Wbf^2) + \frac{T^4}{4}\tr(\Wbf)^2 + o(T^4).
        \end{aligned}
    \end{equation*}
    By plugging them back,
    \begin{equation*}
        \E[\inpr{\gammabf^i}{\Xbf^\top\WQK\xbf_T}^2]
        = \left(\frac{7}{12}+\frac{2i^2}{T^2}\right)\tr(\Wbf^2) + \left(\frac{1}{4}-\frac{i}{T}+\frac{i^2}{T^2}\right)\tr(\Wbf)^2 + o(1).
    \end{equation*}
    Hence, the variance is derived:
    \begin{equation*}
        \begin{aligned}
            v^i
            &= \Vbb[\inpr{\gammabf^i}{\omegabf}] \\
            &= \frac{1}{\lambda^2}\E[\inpr{\gammabf^i}{\Xbf^\top\WQK\xbf_T}^2] - (\mu^i)^2 \\
            &= \frac{1}{\lambda^2}\left(\frac{7}{12}+\frac{2i^2}{T^2}\right)\tr(\Wbf^2) + o(1).
        \end{aligned}
    \end{equation*}
\end{proof}

\sppprop*

\begin{proof}
    To see 1:
    We first see that as $\xi \to \infty$,
    \[
        \Phi\left(\Big(\theta - \frac{1}{2}\Big)\xi; \theta\right) \to
        \begin{cases}
            \frac{1}{2} & \text{if } \theta > \frac{1}{2} \\
            0 & \text{if } \theta = \frac{1}{2} \\
            -\frac{1}{2} & \text{if } \theta < \frac{1}{2}
        \end{cases}
        .
    \]
    In addition, as $\xi \to \infty$ and $\eta \to 0$ with $\xi\eta \to r  \in [0, 2]$,
    \[
        \Phi\left(\Big(\theta - \frac{1}{2}\Big)\xi - \frac{1}{\eta}; \theta\right)
        \to \frac{1}{2}\erf\left(\frac{(\theta-\frac{1}{2})r-1}{\eta\sqrt{2(2\theta^2+\frac{7}{12})}}\right)
        \to -\frac{1}{2}.
    \]
    By combining them, $\rho(\theta) \to \indicator{\theta \ge \frac{1}{2}}$ at the limit.
    If $r > 2$,
    \[
        \Phi\left(\Big(\theta - \frac{1}{2}\Big)\xi - \frac{1}{\eta}; \theta\right)
        \to \frac{1}{2}\erf\left(\frac{(\theta-\frac{1}{2})r-1}{\eta\sqrt{2(2\theta^2+\frac{7}{12})}}\right)
        \to \begin{cases}
            -\frac{1}{2} & \text{if } \theta < \frac{1}{2}+\frac{1}{r} \\
            0 & \text{if } \theta = \frac{1}{2}+\frac{1}{r} \\
            \frac{1}{2} & \text{if } \theta > \frac{1}{2}+\frac{1}{r}
        \end{cases}
        ,
    \]
    and $\rho(\theta) \to \indicator{\frac{1}{2} \le \theta \le \frac{1}{2}+\frac{1}{r}}$ at the limit.

    We can see 2 in the same way as 1.

    To see 3:
    First, compute $\rho'(\theta)$ by using $\diff{}{z}\erf(z) = \frac{2}{\sqrt{\pi}}\exp(-z^2)$:
    \[
        \begin{aligned}
            \rho'(\theta)
            &= \frac{1}{\sqrt{\pi}}\exp\left(\!-\frac{((\theta-\frac{1}{2})\xi)^2}{2(2\theta^2+\frac{7}{12})}\right) \diff{}{\theta}\Bigg\{\frac{(\theta-\frac{1}{2})\xi}{\sqrt{2(2\theta^2+\frac{7}{12})}}\Bigg\} - \frac{1}{\sqrt{\pi}}\exp\left(\!-\frac{((\theta-\frac{1}{2})\xi-\frac{1}{\eta})^2}{2(2\theta^2+\frac{7}{12})}\right) \diff{}{\theta}\Bigg\{\frac{(\theta-\frac{1}{2})\xi-\frac{1}{\eta}}{\sqrt{2(2\theta^2+\frac{7}{12})}}\Bigg\} \\
            &= \left[\frac{1}{\sqrt{\pi}}\exp\left(\!-\frac{((\theta-\frac{1}{2})\xi)^2}{2(2\theta^2+\frac{7}{12})}\right) \frac{4\theta^2-\theta+\frac{5}{3}}{(2(2\theta^2+\frac{7}{12}))^{3/2}} - \frac{1}{\sqrt{\pi}}\exp\left(\!-\frac{((\theta-\frac{1}{2})\xi-\frac{1}{\eta})^2}{2(2\theta^2+\frac{7}{12})}\right) \frac{4\theta^2-\theta+\frac{5}{3}-\frac{1}{\eta}}{(2(2\theta^2+\frac{7}{12}))^{3/2}}\right]\xi.
        \end{aligned}
    \]
    By noting that $0 < \exp(-z^2) \le 1$,
    \[
        \begin{aligned}
            |\rho'(\theta)|
            &\le \frac{|\xi|}{\sqrt{\pi}} \left|\frac{4\theta^2-\theta+\frac{5}{3}}{(2(2\theta^2+\frac{7}{12}))^{3/2}} - \frac{4\theta^2-\theta+\frac{5}{3}-\frac{1}{\eta}}{(2(2\theta^2+\frac{7}{12}))^{3/2}}\right| \\
            &= \frac{|\xi|}{\sqrt{\pi}} \frac{1}{(2(2\theta^2+\frac{7}{12}))^{3/2}\eta} \\
            &\to 0 \text{~~as~~} |\xi| \to 0.
        \end{aligned}
    \]

    To see 4:
    For finite $\xi$,
    \[
        \lim_{\eta \to \infty} \Phi\left(\Big(\theta-\frac{1}{2}\Big)\xi - \frac{1}{\eta}\right) = \Phi\left(\Big(\theta-\frac{1}{2}\Big)\xi\right),
    \]
    which indicates that $\rho(\theta) \to 0$ at the limit $\eta \to \infty$.
\end{proof}

\section{Additional experiments}
\label{section:additional_experiments}

\begin{figure*}
    \centering
    \includegraphics[width=0.35\textwidth]{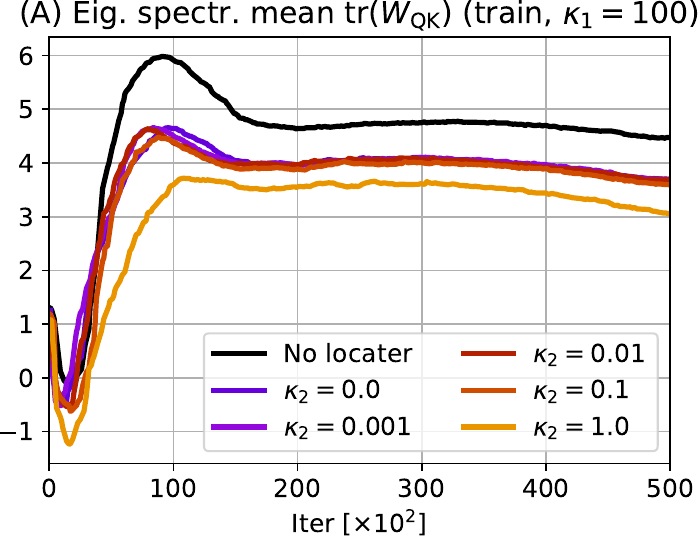}
    \includegraphics[width=0.35\textwidth]{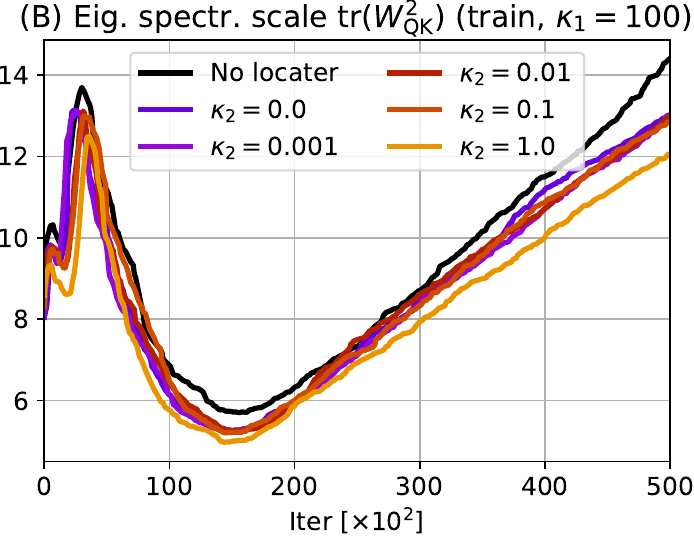} \\
    \includegraphics[width=0.35\textwidth]{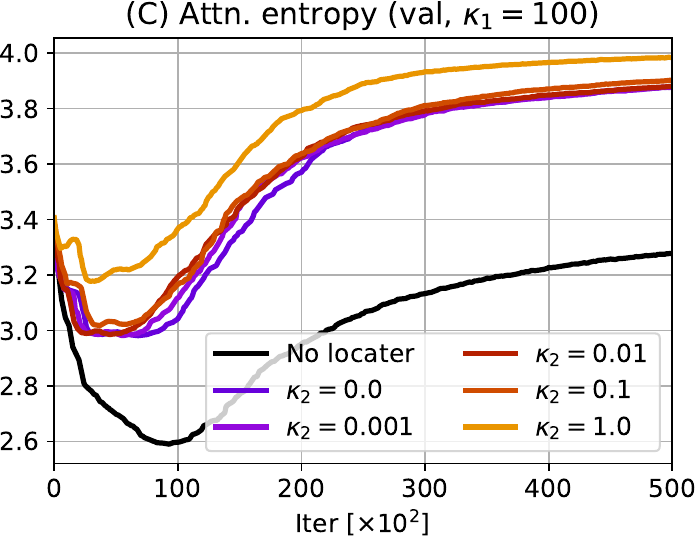}
    \includegraphics[width=0.35\textwidth]{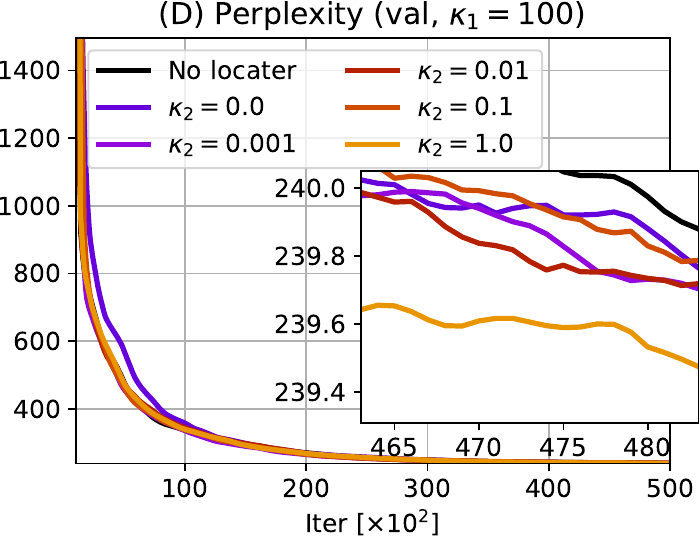}
    \caption{
        Experimental results of language modeling (WikiText-2) with $d=32$ with $1$-layers transformers, fixed $\kappa_1=100$, and varying regularization intensity $\kappa_2$.
        With stronger $\kappa_2$, the eigenspectrum scale shrinks \textbf{(B)}, the attention entropy increases \textbf{(C)}, and the perplexity improves \textbf{(D)}.
    }
    \label{figure:result_lm_d32_l1}
\end{figure*}

\begin{figure*}
    \centering
    \includegraphics[width=0.35\textwidth]{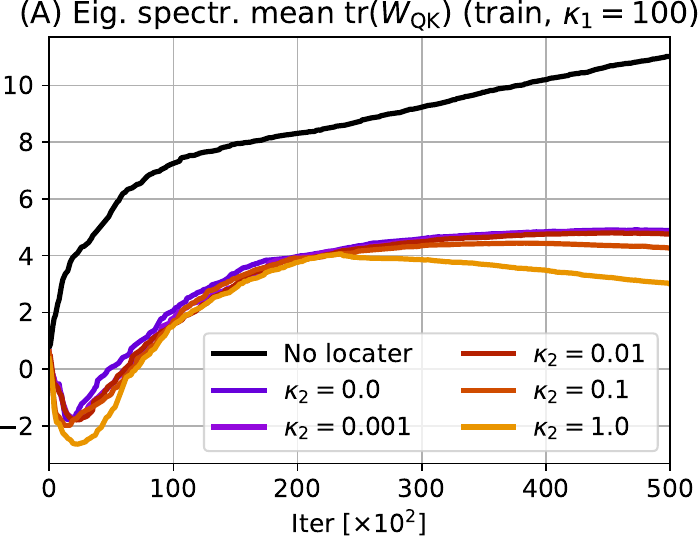}
    \includegraphics[width=0.35\textwidth]{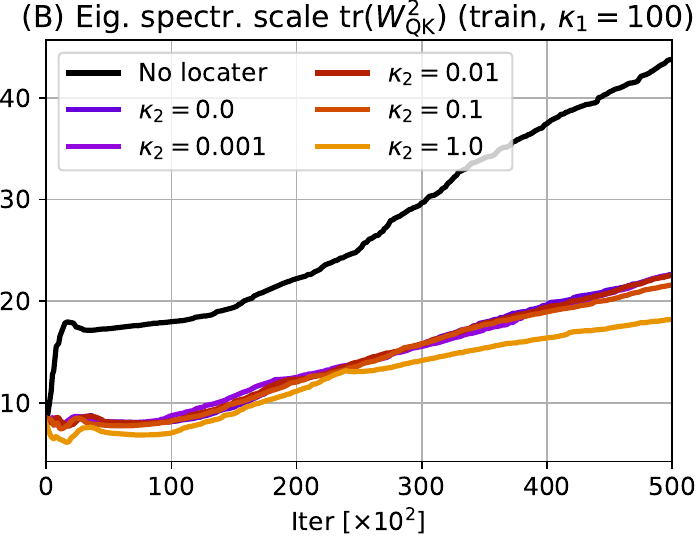} \\
    \includegraphics[width=0.35\textwidth]{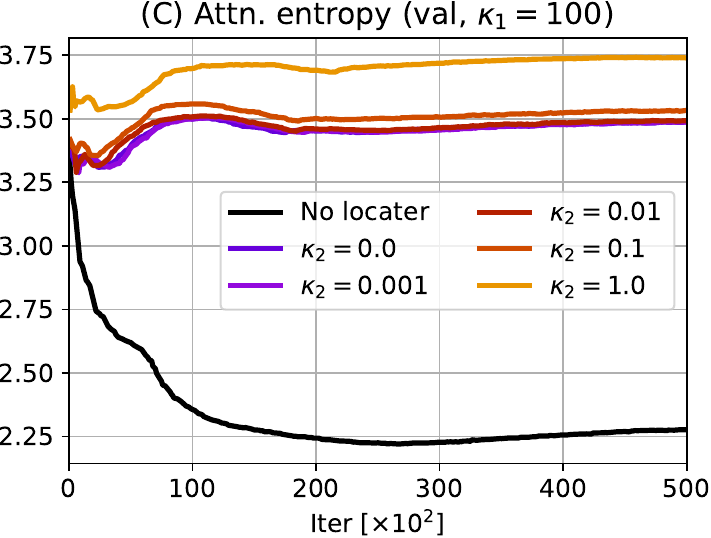}
    \includegraphics[width=0.35\textwidth]{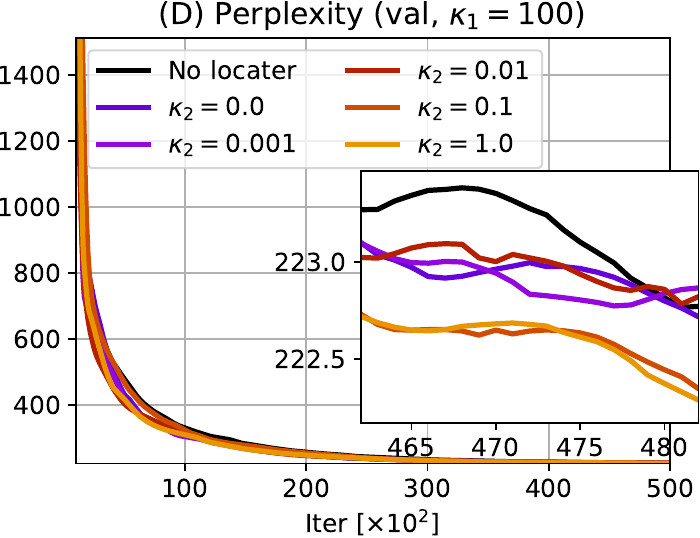}
    \caption{
        Experimental results of language modeling (WikiText-2) with $d=32$ with $3$-layers transformers, fixed $\kappa_1=100$, and varying regularization intensity $\kappa_2$.
        With stronger $\kappa_2$, the eigenspectrum scale shrinks \textbf{(B)}, the attention entropy increases \textbf{(C)}, and the perplexity improves \textbf{(D)}.
    }
    \label{figure:result_lm_d32_l3}
\end{figure*}

\begin{figure*}
    \centering
    \includegraphics[width=0.35\textwidth]{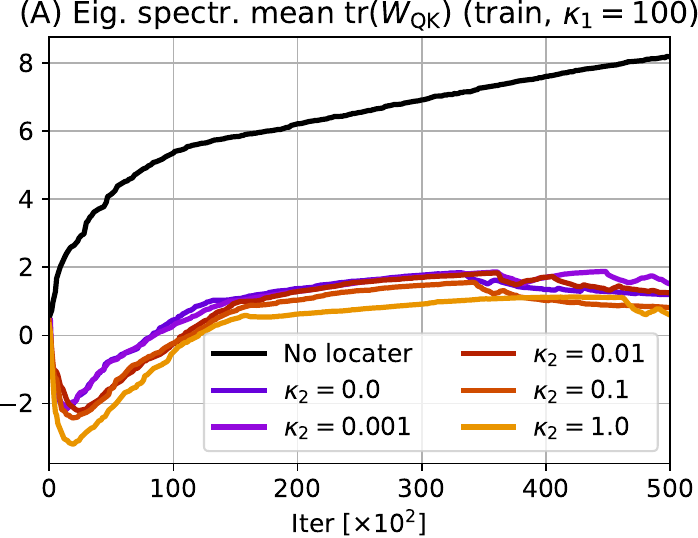}
    \includegraphics[width=0.35\textwidth]{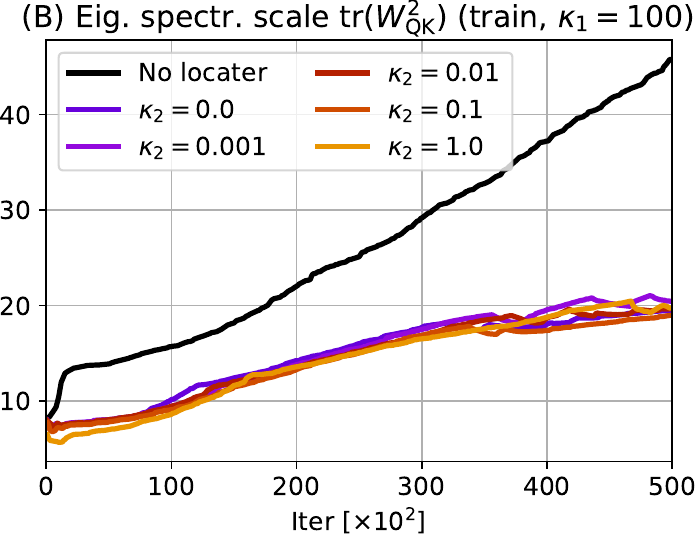} \\
    \includegraphics[width=0.35\textwidth]{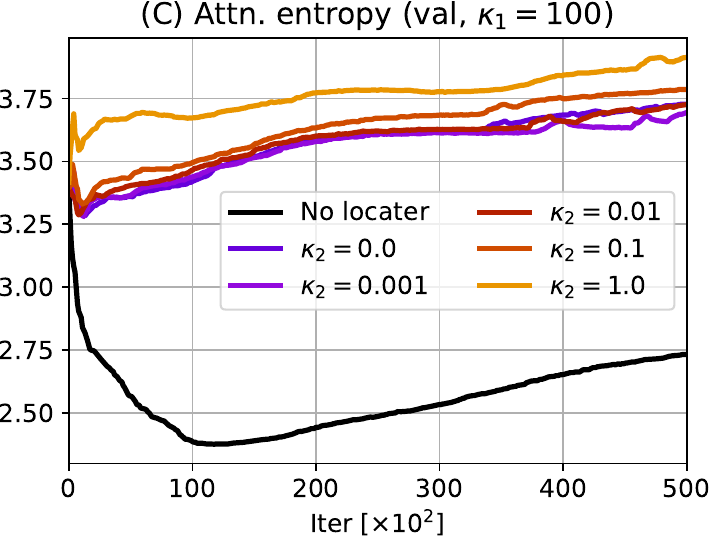}
    \includegraphics[width=0.35\textwidth]{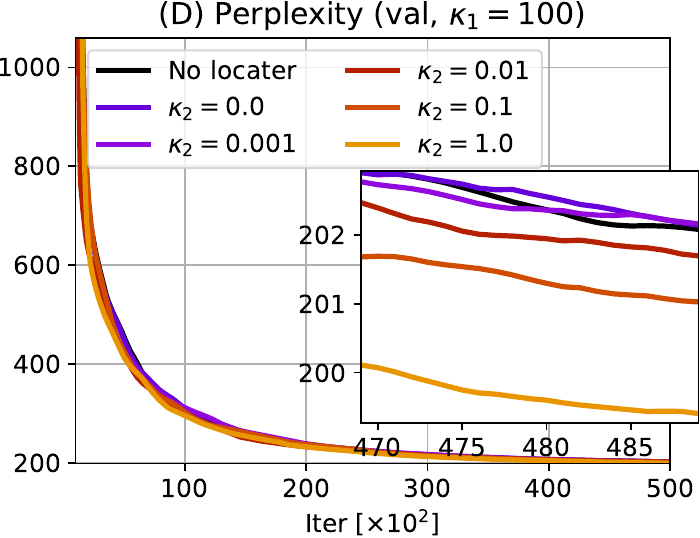}
    \caption{
        Experimental results of language modeling (WikiText-2) with $d=32$ with $6$-layers transformers, fixed $\kappa_1=100$, and varying regularization intensity $\kappa_2$.
        With stronger $\kappa_2$, the eigenspectrum scale shrinks \textbf{(B)}, the attention entropy increases \textbf{(C)}, and the perplexity improves \textbf{(D)}.
    }
    \label{figure:result_lm_d32_l6}
\end{figure*}

\begin{figure*}
    \centering
    \includegraphics[width=0.35\textwidth]{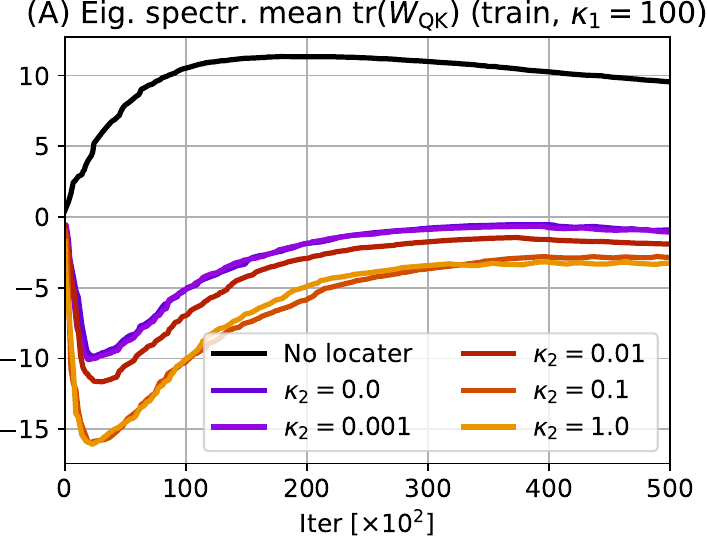}
    \includegraphics[width=0.35\textwidth]{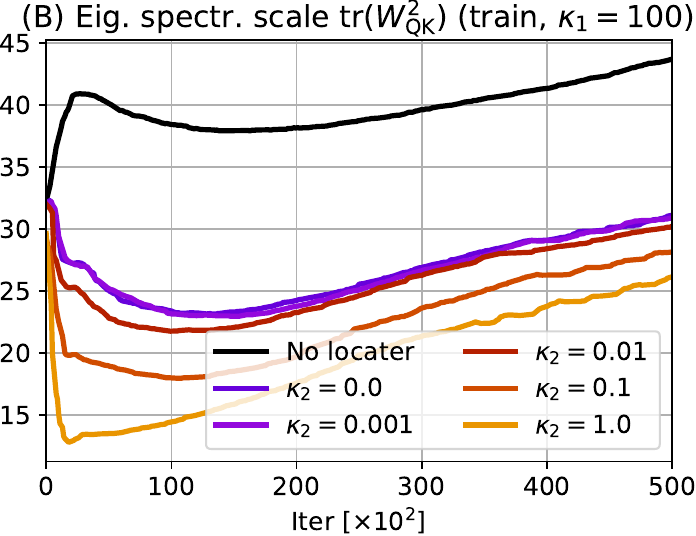} \\
    \includegraphics[width=0.35\textwidth]{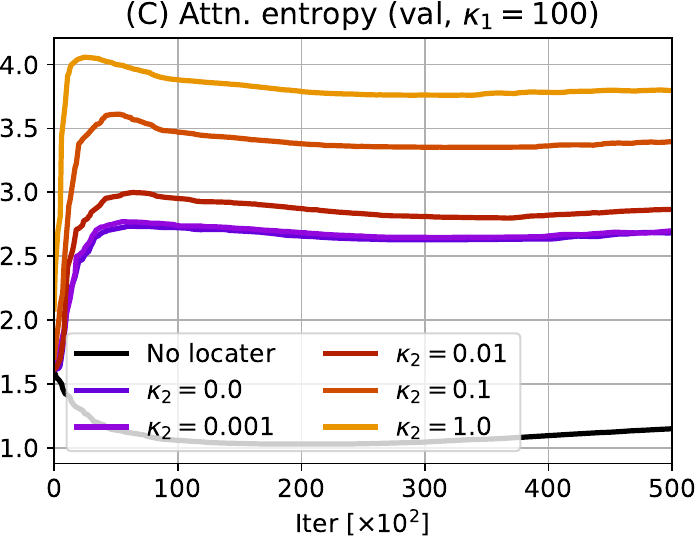}
    \includegraphics[width=0.35\textwidth]{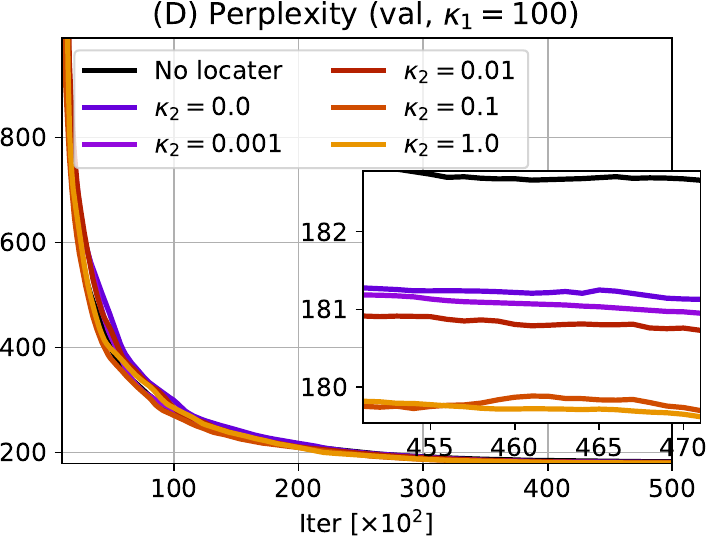}
    \caption{
        Experimental results of language modeling (WikiText-2) with $d=128$ with $3$-layers transformers, fixed $\kappa_1=100$, and varying regularization intensity $\kappa_2$.
        With stronger $\kappa_2$, the eigenspectrum scale shrinks \textbf{(B)}, the attention entropy increases \textbf{(C)}, and the perplexity improves \textbf{(D)}.
    }
    \label{figure:result_lm_d128_l3}
\end{figure*}

\begin{figure*}
    \centering
    \includegraphics[width=0.35\textwidth]{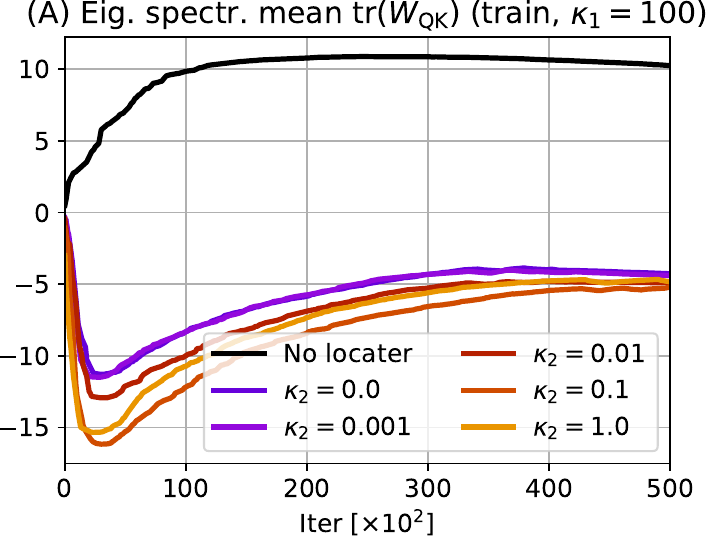}
    \includegraphics[width=0.35\textwidth]{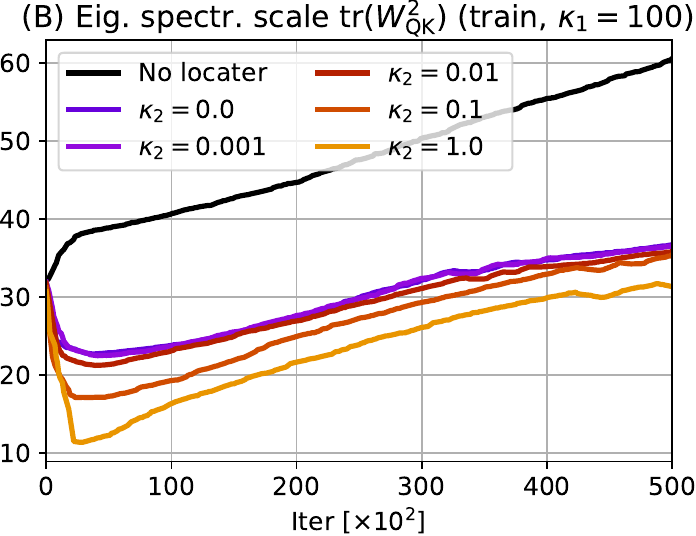} \\
    \includegraphics[width=0.35\textwidth]{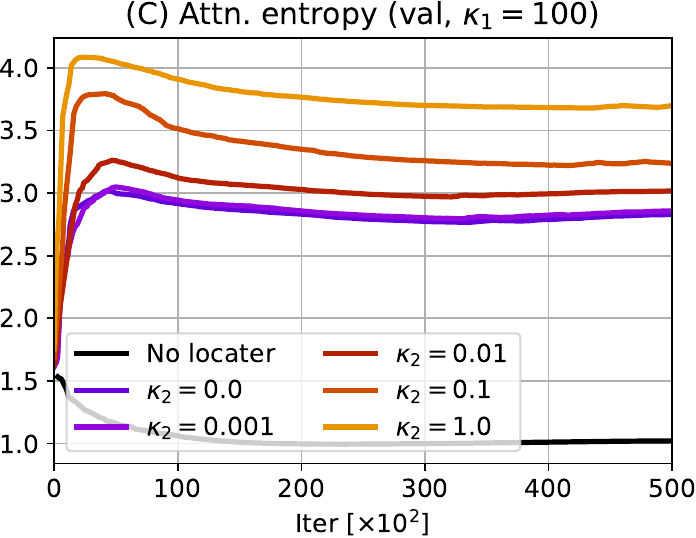}
    \includegraphics[width=0.35\textwidth]{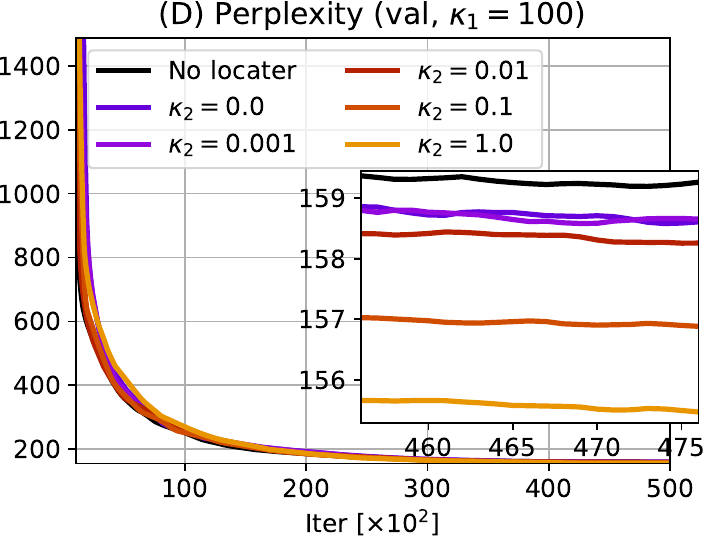}
    \caption{
        Experimental results of language modeling (WikiText-2) with $d=128$ with $6$-layers transformers, fixed $\kappa_1=100$, and varying regularization intensity $\kappa_2$.
        With stronger $\kappa_2$, the eigenspectrum scale shrinks \textbf{(B)}, the attention entropy increases \textbf{(C)}, and the perplexity improves \textbf{(D)}.
    }
    \label{figure:result_lm_d128_l6}
\end{figure*}

Here, we show additional results of the language modeling task with $1$-/$3$-/$6$-layer transformers with different embedding dimensions $d=32,128$.
For $d=128$, the configurations remain the same except for the number of decoder layers as in \cref{section:experiments}.
For $d=32$, we used the learning rate $0.0001$ (instead of $0.000025$ used for $d=128$), and the other configurations remain the same.
The results are shown in \cref{figure:result_lm_d32_l1} ($d=32$, $1$-layers), \cref{figure:result_lm_d32_l3} ($d=32$, $3$-layers), \cref{figure:result_lm_d32_l6} ($d=32$, $6$-layers), \cref{figure:result_lm_d128_l3} ($d=128$, $3$-layers), and \cref{figure:result_lm_d128_l6} ($d=128$, $6$-layers).
The overall trends are quite similar to the case of $1$-layer transformers with $d=128$ as seen in \cref{figure:result_lm_d128_l1}:
As $\kappa_2$ increases, the eigenspectrum scale decreases, the attention entropy increases, and eventually, the perplexity improves (namely, decreases).

\end{document}